\documentclass[sigconf]{aamas} 

\usepackage{balance} 
\usepackage{imports}

\setcopyright{ifaamas}
\acmConference[AAMAS '24]{Proc.\@ of the 23rd International Conference
on Autonomous Agents and Multiagent Systems (AAMAS 2024)}{May 6 -- 10, 2024}
{Auckland, New Zealand}{N.~Alechina, V.~Dignum, M.~Dastani, J.S.~Sichman (eds.)}
\copyrightyear{2024}
\acmYear{2024}
\acmDOI{}
\acmPrice{}
\acmISBN{}

\acmSubmissionID{346}

\title[Safe Model-Based Multi-Agent Mean-Field Reinforcement Learning]{Safe Model-Based Multi-Agent \\ Mean-Field Reinforcement Learning}

\author{Matej Jusup}
\affiliation{
  \institution{ETH Zurich}
  \city{Zurich}
  \country{Switzerland}
  }
\email{mjusup@ethz.ch}

\author{Barna P\'asztor}
\affiliation{
  \institution{ETH Zurich}
  \city{Zurich}
  \country{Switzerland}}
\email{barna.pasztor@ai.ethz.ch}

\author{Tadeusz Janik}
\affiliation{
  \institution{ETH Zurich}
  \city{Zurich}
  \country{Switzerland}}
\email{tjanik@student.ethz.ch}

\author{Kenan Zhang}
\affiliation{
  \institution{EPFL Lausanne}
  \city{Lausanne}
  \country{Switzerland}}
\email{kenan.zhang@epfl.ch}

\author{Francesco Corman}
\affiliation{
  \institution{ETH Zurich}
  \city{Zurich}
  \country{Switzerland}}
\email{corman@ethz.ch}

\author{Andreas Krause }
\affiliation{
  \institution{ETH Zurich}
  \city{Zurich}
  \country{Switzerland}}
\email{krausea@ethz.ch}

\author{Ilija Bogunovic}
\affiliation{
  \institution{University College London}
  \city{London}
  \country{United Kingdom}}
\email{i.bogunovic@ucl.ac.uk}

\begin{abstract}
Many applications, e.g., in shared mobility, require coordinating a large number of agents.  Mean-field reinforcement learning addresses the resulting scalability challenge by optimizing the policy of a representative agent interacting with the infinite population of identical agents instead of considering individual pairwise interactions. In this paper, we address an important generalization where there exist global constraints on the distribution of agents (e.g., requiring capacity constraints or minimum coverage requirements to be met). We propose \algnme, the first model-based mean-field reinforcement learning algorithm that attains safe policies even in the case of \textit{unknown} transitions. As a key ingredient, it uses epistemic uncertainty in the transition model within a log-barrier approach to ensure pessimistic constraints satisfaction with high probability. Beyond the synthetic swarm motion benchmark, we showcase \algnmespace on the vehicle repositioning problem faced by many shared mobility operators and evaluate its performance through simulations built on vehicle trajectory data from a service provider in Shenzhen. Our algorithm effectively meets the demand in critical areas while ensuring service accessibility in regions with low demand.
\end{abstract}


\newcommand{\BibTeX}{\rm B\kern-.05em{\sc i\kern-.025em b}\kern-.08em\TeX}

\begin{document}


\pagestyle{fancy}
\fancyhead{}


\maketitle 


\begin{figure}[t!]
\centering
\vspace{0.4cm}
 \includegraphics[width=\columnwidth]{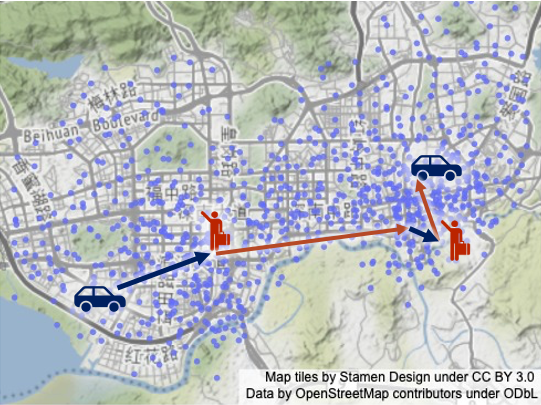}
 \caption{An illustration of vehicles' spatial distribution (light-blue scatters), repositioning trips (blue arrows), and a trajectory of passenger trips (red arrows). }
 \label{fig:trajectory}
 \Description{An illustration of vehicles' spatial distribution (light-blue scatters), repositioning trips (blue arrows), and a trip trajectory as a sequence of passenger trips (red arrows)}
\end{figure} 

\vspace{-1cm}
\section{Introduction}\label{sec:intro}
Multi-Agent Reinforcement Learning (MARL) is a rapidly growing field that seeks to understand and optimize the behavior of multiple agents interacting in a shared environment. MARL has a wide range of potential applications, including vehicle repositioning in shared mobility services (e.g., moving idle vehicles from low-demand to high-demand areas~\cite{lin2018efficient}), swarm robotics (e.g., operating a swarm of drones~\cite{alon2020multi}), and smart grids (e.g., operating a network of sensors in electric system~\cite{roesch}).
The interactions between agents in these complex systems introduce several challenges, including non-stationarity, scalability, competing learning goals, and varying information structure.
Mean-Field Control (MFC) addresses the scalability and non-stationarity hurdles associated with MARL by exploiting the insight that many relevant MARL problems involve a large number of very similar agents working towards the same goal.
Instead of focusing on the individual agents and their interactions, MFC considers an asymptotically large population of identical cooperative agents and models them as a distribution on the state space. This approach circumvents the problem's dependency on the population size, enabling the consideration of large populations.
The solutions obtained by MFC are often sufficient for the finite-agent equivalent problem~\cite{lacker2017limit,mondal2022mean,wang2020breaking,chen2021pessimism} in spite of the introduced approximations.
An example of such a system is a ride-hailing platform that serves on-demand trips with a fleet of vehicles. The platform needs to proactively reposition idle vehicles based on their current locations, the locations of the other vehicles in the fleet, and future demand patterns (see \Cref{fig:trajectory}) to maximize the number of fulfilled trips and minimize customer waiting times. Additionally, the platform may be obligated by external regulators to guarantee service accessibility across the entire service region. The problem quickly becomes intractable as the number of vehicles increases. A further difficulty lies in modeling the traffic flows. Due to numerous infrastructure, external, and driver behavioral factors, which are often region-specific, it is laborious and often difficult to determine transitions precisely \cite{daganzo1994cell,van2018link,bliemer2020static}.    

In this paper, we focus on learning the \textit{safe optimal policies} for a large multi-agent system when the underlying transitions are \textit{unknown}. In most real-world systems, the transitions must be learned from the data obtained from repeated interactions with the environment. We assume that the cost of obtaining data from the environment is high and seek to design a model-based solution that efficiently uses the collected data. Existing works consider solving the MFC problem via model-free or model-based methods without safety guarantees.
However, the proposed \textit{Safe Model-Based Multi-Agent Mean-Field Upper-Confidence Reinforcement Learning}  (\algnme) algorithm focuses on learning underlying transitions and deriving optimal policies for the mean-field setting while avoiding undesired or dangerous distributions of agents' population.\looseness=-1

\textbf{Contributions.} 
\Cref{sec:math} extends the MFC setting with safety constraints and defines a novel comprehensive notion of global population-based safety. To address safety-constrained environments featuring a large population of agents, in \Cref{sec:proposed_method}, we propose a model-based mean-field reinforcement learning algorithm called \algnme. Our algorithm leverages epistemic uncertainty in the transition model, employing a log-barrier approach to guarantee pessimistic satisfaction of safety constraints and enables the derivation of safe policies. In \Cref{sec:experiments}, we conduct empirical testing on the synthetic swarm motion benchmark and real-world vehicle repositioning problem, a challenge commonly faced by shared mobility operators. Our results demonstrate that the policies derived using \algnmespace successfully fulfill demand in critical demand hotspots while ensuring service accessibility in areas with lower demand.\looseness=-1

\section{Related Work}
Our notion of safety for the mean-field problem extends the frameworks of \textit{Mean-Field Games} (MFG) and \textit{Mean-Field Control} (MFC)~\cite{lasry2006jeux1,lasry2006jeux2, huang2006large, huang2007large}. For a summary of the progress focusing on MFGs, see~\cite{LauriereLearningSurvey} and references therein. We focus on MFCs in this work, which assume cooperative agents in contrast to MFGs, which assume competition.~\cite{gast2012mean, Motte2019Mean-fieldControls, Gu2019DynamicMFCs,  Gu2021Mean-FieldAnalysis, Bauerle2021MeanProcesses} address the problem of solving MFCs under known transitions, i.e., planning, while~\cite{wang2020breaking, carmona2019model, Angiuli2022UnifiedProblems, Angiuli2021ReinforcementEconomics, Carmona2019Linear-QuadraticMethods, Wang2021GlobalTime, Carmona2021Linear-quadraticOptimization, subramanian_reinforcement_2019} consider model-free Q-learning and Policy Gradient methods in various settings. Closest to our approach,~\cite{pasztor2021efficient} introduces \ucrl, a model-based, on-policy algorithm, which is more sample efficient than other proposed approaches. Similarly to~\cite{jaksch2010near, chowdhury2019online, curi2020efficient} for model-based single-agent RL and~\cite{Sessa2022EfficientComputation} for model-based MARL, \ucrlspace uses the epistemic uncertainty in the transition model to design optimistic policies that efficiently balance exploration and exploitation in the environment and maximize sample efficiency. This is also the setting of our interest. However, safety is not considered in any of these methods.\looseness=-1

In terms of \textit{safety}, there are two main ways of handling it in RL; assigning significantly lower rewards to unsafe states~\cite{moldovan2012safe} and providing additional knowledge to the agents~\cite{song2012efficient} or using the notion of controllability to avoid unsafe policies explicitly~\cite{gehring2013smart}. 
Furthermore, the following approaches combine the two methods;~\cite{berkenkamp2017safe} uses Lyapunov functions to restrict the safe policy space,~\cite{cheng2019end} projects unsafe policies to a safe set via a control barrier function, and~\cite{alshiekh2018safe} introduces shielding, i.e., correcting actions only if they lead to unsafe states.
For comprehensive overviews on safe RL, we refer the reader to~\cite{garcia2015comprehensive, gu2022review}.
As an alternative, \cite{usmanova2022log} demonstrates that the general-purpose stochastic optimization methods can be used for constrained MDPs, i.e., safe RL formulations. Similar to our work, they use the log-barrier approach to turn constrained into unconstrained optimization.  
Nevertheless, the aforementioned works focus mainly on individual agents, while in large-scale multi-agent environments, maintaining individual safety becomes intractable, and the focus shifts towards global safety measures.
For multi-agent problems, previous works focus on satisfying the individual constraints of the agents while learning in a multi-agent environment.
For the cooperative problem,~\cite{gu2021multi} proposes two model-free algorithms, MACPO and MAPPO-Lagrangian. MACPO is computationally expensive, while MAPPO-Lagrangian does not guarantee hard constraints for safety.
Dec-PG solves the decentralized learning problem using a consensus network that shares weights between neighboring agents~\cite{lu2021decentralized}.
For the non-cooperative decentralized MARL problem with individual constraints,~\cite{sheebaelhamd2021safe} adds a safety layer to multi-agent DDPG~\cite{lowe2017multi} similar to single-agent Safe DDPG~\cite{dalal2018safe} for continuous state-action spaces.
Aggregated and population-based constraints have been addressed in the following works.
CMIX~\cite{liu2021cmix} extends QMIX~\cite{rashid2020monotonic}, which considers average and peak constraints defined over the whole population of agents in a centralized-learning decentralized-execution framework. Their formulation relies on the joint state and action spaces, making it infeasible for a large population of agents.
\cite{elsayed2021safe} introduces an additional shielding layer that corrects unsafe actions. Their centralized approach suffers from scalability issues, while the factorized shielding method monitors only a subset of the state or action space.
For mixed cooperative-competitive settings,~\cite{zhu2020multi} uses the notion of returnability to define a safe, decentralized, multi-agent version of Q-learning that ensures individual and joint constraints. However, their approach requires an estimation of other agents' policies, which does not scale well for large systems.
Works considering constraints on the whole population fail to overcome the exponential nature of multi-agent problems or require domain knowledge to factorize the problems into subsets.
 
Closest to our setting,~\cite{mondal2022mean} introduces constraints to the MFC by defining a cost function and a threshold that the discounted sum of costs can not exceed. 
We propose a different formulation that restricts the set of feasible mean-field distributions at every step, therefore, addressing the scalability issue and allowing for more specific control over constraints and safe population distributions.

\section{Problem Statement}\label{sec:math}
Formally, we consider the \textit{episodic} setting, where episodes $n {}={}1, \ldots, N$ each have $t {}={}0, \ldots, T-1$ discrete steps and the terminal step $t=T$. The state space $\s \subseteq \R^p$ and action space $\A \subseteq \R^q$ are the same for every agent. We use $s_{n,t}^{(i)} \in \s$ and $a_{n,t}^{(i)} \in \A$, to denote the state and action of agent $i \in \{1, \ldots, m\}$ in episode $n$ at step $t$. For every $n$ and $t$, the \textit{mean-field distribution} $\mu_{n,t} \in \pp(\s)$ describes the global state with $m$ identical agents when $m \rightarrow {}+{}\infty$, i.e.,
\begin{equation*}
    \mu_{n,t}(ds) {}={}\lim_{m \xrightarrow{} \infty} \frac{1}{m} \sum_{i=1}^m \I(s_{n,t}^{(i)}\in ds),
\end{equation*}
where $\I(\cdot)$ is the indicator function, and $\pp(\s)$ is the set of probability measures over the state space $\s$.\looseness=-1

We consider the MFC model to capture a collective behavior of a \textit{large} number of \textit{collaborative} agents operating in the shared \textit{stochastic environment}. This model assumes the limiting regime of \textit{infinitely} many agents and \textit{homogeneity}. Namely, all agents are identical and indistinguishable, therefore, solving MFC amounts to finding an optimal policy for a single, so-called, \textit{representative agent}.
The representative agent interacts with the mean-field distribution of agents instead of focusing on individual interactions and optimizes a collective reward. Due to the homogeneity assumption, the representative agent's policy is used to control all the agents in the environment. \looseness=-1

We posit that the reward $r: \s \times \pp(\s) \times \A \rightarrow \mathbb{R}$ of the representative agent is known and that it depends on the states of the other agents through the mean-field distribution.\footnote{Our framework easily extends to unknown reward by estimating its epistemic uncertainty and learning it similarly to learning unknown transitions (see~\cite{chowdhury2019online}).}
Before every episode $n$, the representative agent selects a non-stationary policy profile $\bpi_n {}={}(\pi_{n,0},\ldots,\pi_{n,T-1}) \in \Pi$ where individual policies
are of the form $\pi_{n,t}: \s \times \pp(\s) \rightarrow \A$ and $\Pi$ is the set of admissible policy profiles. The policy profile is then shared with all the agents that choose their actions according to $\bpi_n$ during episode $n$. \looseness=-1

We consider a general family of deterministic transitions $f:\s \times \pp(\s) \times \A \rightarrow \s$. Given the current mean-field distribution $\mu_{n,t}$, the representative agent's state $s_{n,t}$ and its action $a_{n,t}$, the next representative agent's state $s_{n,t{}+{}1}$ is given by
\begin{equation} \label{eqn:transitions}
    s_{n,t{}+{}1} {}={}f(s_{n,t}, \mu_{n,t}, a_{n,t}) {}+{} \epsilon_{n,t},
\end{equation}
where $\epsilon_{n,t}$ is a Gaussian noise with known variance. We assume that the transitions are \textit{unknown} and are to be inferred from collected trajectories across episodes. 

\textbf{Mean-field transitions.} 
State-to-state transition map in \Cref{eqn:transitions} naturally extends to the \textit{mean-field transitions} induced by a policy profile $\bpi_n$ and transitions $f$  in episode $n$ (see~\cite[Lemma 1]{pasztor2021efficient}) \looseness=-1
\begin{equation} \label{eqn:mf_transition}
    \mu_{n,t{}+{}1}(ds') {}={}\int_\s \p[s_{n,t{}+{}1} \in ds']\mu_{n,t}(ds) ,
\end{equation}
where $s_{n,t{}+{}1}$ is the next representative agent state and $\mu_{n,t}(ds) {}={}\p[s_{n,t} \in ds]$ under $\bpi_n$ for all $t$.
To simplify the notation, we use $U(\cdot)$ to denote the mean-field transition function from \Cref{eqn:mf_transition}, i.e., we have $\mu_{n,t{}+{}1} {}={}U(\mu_{n,t}, \pi_{n,t}, f)$. We further introduce the notation $z_{n,t} \in \zz {}={} \s \times \pp(\s) \times \A$ to denote the tuple $(s_{n,t}, \mu_{n,t}, a_{n,t})$.

For a given policy profile $\bpi_n$ and mean-field distribution $\mu$, the performance of the representative agent at step $t$ is measured via its expected future reward for the rest of the episode, i.e., $$\E \left[\sum_{j=t}^{T-1} r(z_{n,j}) {}\rvert{} \mu_{n,t} {}={}\mu\right].$$
Here, the expectation is over the randomness in the transitions and over the sampling of the initial state, i.e., $s_{n,t} \sim \mu$. \looseness=-1

\subsection{Safe Mean-Field Reinforcement Learning}
We extend the MFC with global safety constraints,\footnote{For the exposition, we use a single constraint, however, our approach is directly extendable to multiple constraints.} i.e., the constraints imposed on the mean-field distributions. We consider safety functions $h: \pp(\s) \to \mathbb{R}$ over probability distributions. Given some hard safety threshold $C \in \mathbb{R}$, we consider a mean-field distribution $\mu$ as safe if it satisfies $h(\mu) {}\geq{} C$, or, equivalently \looseness=-1 
\begin{equation} \label{eqn:safety}
    h_C(\mu) {}:={}h(\mu) - C {}\geq{} 0.
\end{equation}
We denote the set of safe mean-field distributions for a safety constraint $h_C(\cdot)$ as $\zeta {}={}\{ \mu \in \pp(\s): h_C(\mu) {}\geq{} 0 \}$. Hence, our focus is on the safety of the system as a whole rather than the safety of individual agents, as it becomes intractable to handle individual agents' states and interactions in the case of a large population. \looseness=-1

For a given initial distribution $\mu_0$, we formally define the \textit{Safe-MFC}\footnote{We refer to formulations under known transitions as control problems, while we reserve the term reinforcement learning for formulations under unknown transitions.} problem as follows
\begin{subequations}\label{eqn:safe-mf-rl}
\begin{align}
    \bpi^* {}={}\argmax_{\bpi \in \Pi} {}& \E \Bigg [ \sum_{t=0}^{T-1} r(z_t) \Big\rvert \mu_{0}\Bigg ] \\
    \text{subject to} \quad a_{t} &{}={}\pi_t(s_{t}, \mu_t) \\
    s_{t{}+{}1} &{}={}f(z_{t}) {}+{} \epsilon_{t}  \label{eqn:safe-mf-rl-transition}\\
    \mu_{t{}+{}1} &{}={}U(\mu_t, \pi_t, f) \label{eqn:safe-mf-rl-mf-transition} \\
    h_C(\mu_{t{}+{}1}) &{}\geq{} 0, \label{eqn:safe-mf-rl-safety}
\end{align}
\end{subequations}
where we explicitly require induced mean-field distributions $\{\mu_t\}_{t=1}^T$ to reside in the safe set $\zeta$ by restricting the set of admissible policy profiles $\Pi$ to policy profiles that induce safe distributions.
To ensure complete safety, we note that the initial mean-field distribution \textit{$\mu_0$ must be in the safe set $\zeta$ as our learning protocol does not induce it} (see \Cref{alg:learning_protocol}). \looseness=-1

We make the following assumptions about the environment using Wasserstein $1$-distance defined by $$W_1(\mu, \mu') {}:={}\inf_{\gamma \in \Gamma(\mu, \mu')} \E_{(x,y)\sim\gamma}\|x - y \|_1,$$
where $\Gamma(\mu, \mu')$ is the set of all couplings of $\mu$ and $\mu'$ (i.e., a joint probability distributions with marginals $\mu$ and $\mu'$). We further define the distance between $z=(s, \mu, a)$ and $z'=(s', \mu', a')$ as
$$d(z,z'){}:={}\|s-s'\|_2 {}+{} \|a-a'\|_2 {}+{} W_1(\mu, \mu').$$ 

\begin{assumption}[Transitions Lipschitz continuity]
\label{asm:transition_lipschitz}
The transition function $f(\cdot)$ is $L_f$-Lipschitz-continuous, i.e., $$\|f(z) - f(z')\|_2 {}\leq{} L_f d(z,z').$$
\end{assumption}

\begin{assumption}[Mean-field policies Lipschitz continuity]
\label{asm:policy_lipschitz}
The individual policies $\pi$ present in any admissible policy profile $\bpi$ in $\Pi$ are $L_{\pi}$-Lipschitz-continuous, i.e., $$\|\pi(s, \mu) - \pi(s', \mu')\|_2 {}\leq{} L_{\pi}(\|s-s'\|_2 {}+{} W_1(\mu, \mu'))$$ for all $\pi \in \bpi \in \Pi$.
\end{assumption}

\begin{assumption}[Reward Lipschitz continuity]
\label{asm:reward_lipschitz}
The reward function $r(\cdot)$ is $L_r$-Lipschitz-continuous, i.e., $$\|r(z) - r(z')\|_2 {}\leq{} L_rd(z,z').$$
\end{assumption}

These assumptions are considered standard in model-based learning~\cite{jaksch2010near, chowdhury2019online, Sessa2022EfficientComputation, pasztor2021efficient} and mild, as individual policies and rewards are typically designed such that they meet these smoothness requirements. For example, we use neural networks with Lipschitz-continuous activations to represent our policies (see \Cref{apx:vehicle_learning_protocol}). \looseness=-1

\subsection{Examples of Safety Constraints} \label{sec:safety_constraints_examples}
We can model multiple classes of safety constraints $h_C(\cdot) {}\geq{} 0$ that naturally appear in real-world applications such as vehicle repositioning, traffic flow, congestion control, and others.

\textbf{Entropic safety.}
Entropic constraints can be used in multi-agent systems to prevent overcrowding by promoting spatial diversity and avoiding excessive clustering. Incorporating an entropic term in the decision-making process encourages the controller to distribute the agents evenly within the state space. This might be particularly useful in applications that include crowd behavior, such as operating a swarm of drones or a fleet of vehicles. In such scenarios, we define safety by imposing a threshold $C \geq 0$ on the differential entropy \looseness=-1
\begin{equation} \label{eqn:differential_entropy}
    H(\mu) {}:={}-\int_\s\log\mu(s)\mu(ds)
\end{equation}
of the mean-field distribution $\mu$, i.e.,
\begin{equation*} \label{eqn:entropic_safety}
    h_C(\mu) {}:={}H(\mu) - C.
\end{equation*}

\textbf{Distribution similarity.}
Another way to define safety is by preventing $\mu$ from diverging from a prior distribution $\nu_0$. The prior can be based on previous studies, expert opinions or regulatory requirements.
We can use a penalty function that quantifies the allowed dissimilarity between the two distributions \looseness=-1
\begin{equation*}
    h_{C}(\mu; \nu_0) {}:={}C - D(\mu, \nu_0),
\end{equation*}
with $C \geq 0$ and where the distance function between probability measures $D:\pp(\s)\times\pp(\s)\rightarrow\R_{\geq 0}$ depends on the problem at hand.

We provide further examples of safety functions in \Cref{apx:proof_useful_constraints} together with proofs that they satisfy \Cref{asm:safety_lipschitz}.

\subsection{Statistical Model and Safety Implications} \label{sec:statistical_model}
The representative agent learns about unknown transitions by interacting with the environment. We take a model-based approach to achieve sample efficiency by sequentially updating and improving the transition model estimates based on the previously observed transitions.
At the beginning of each episode $n$, the representative agent updates its model based on $\cup_{i=1}^{n-1}\D_i$ where $\D_i=\{(z_{i,t},s_{i,t{}+{}1})\}_{t=0}^{T-1}$ and $z_{i,t}=(s_{i,t},\mu_{i,t},a_{i,t})$ is the set of observations in episode $i$ for $i=1,...,n-1$, i.e., up until the beginning of episode $n$.
We estimate the mean $\bM_{n-1}:\zz \rightarrow \s$ and covariance $\bSigma_{n-1}:\zz \rightarrow \R^{p \times p}$ functions from the set of collected trajectories $\cup_{i=1}^{n-1} \D_i$, and denote model's confidence with $\bsigma^2_{n-1}(z)=\text{diag}(\bSigma_{n-1}(z))$. We assume that the statistical model is calibrated, meaning that at the beginning of every episode, the agent has \textit{high probability confidence bounds} around unknown transitions. 
The following assumptions are consistent with~\cite{srinivas2009gaussian, chowdhury2019online, curi2020efficient, Sessa2022EfficientComputation, pasztor2021efficient} and other literature which aims to exclude extreme functionals from consideration.\looseness=-1

\begin{assumption}[Calibrated model]
\label{asm:calibrated_model}
Let $\bM_{n-1}(\cdot)$ and $\bSigma_{n-1}(\cdot)$ be the mean and covariance functions of the statistical model of $f$ conditioned on $n - 1$ observed episodes. For the confidence function $\bsigma_{n-1}(\cdot)$, there exists a non-decreasing, strictly positive sequence $\{\beta_n\}_{n {}\geq{} 0}$ such that for $\delta > 0$ with probability at least $1-\delta$, we have jointly for all $n {}\geq{} 1$ and $z \in \zz$ that $|f(z) - \bM_{n-1}(z)| {}\leq{} \beta_{n-1}\bsigma_{n-1}(z)$ elementwise.
\end{assumption}

\begin{assumption}[Estimated confidence Lipschitz continuity]
\label{asm:calibrated_lipschitz}
The confidence function $\bsigma_n(\cdot)$ is $L_{\sigma}$-Lipschitz-continuous for all $n {}\geq{} 0$, i.e., $\|\bsigma_n(z) - \bsigma_n(z')\|_2 {}\leq{} L_\sigma d(z,z')$.
\end{assumption}

Since the true transition model is unknown in \Cref{eqn:safe-mf-rl-transition} and \Cref{eqn:safe-mf-rl-mf-transition}, at the beginning of every episode $n$, the representative agent can only construct the confidence set of transitions $\F_{n-1}$ with $\bM_{n-1}(\cdot)$ and $\bsigma_{n-1}(\cdot)$ estimated based on the observations up until the end of the previous episode $n-1$, i.e.,\looseness=-1
\begin{align}
        \F_{n-1} {}={}\Big\lbrace \tf: \tf \text{ is calibrated w.r.t. } \bM_{n-1}(\cdot) \text{ and } \bsigma_{n-1}(\cdot) \Big\rbrace \label{eq:model_class}
\end{align}
The crucial challenge in \textit{Safe Mean-Field Reinforcement Learning} is that the representative agent can only select transitions $\tf \in \F_{n-1}$ at the beginning of the episode $n$ and use it instead of true transitions $f$ when solving \Cref{eqn:safe-mf-rl} to find an optimal policy profile $\bpi_n^*$. The resulting mean-field distributions $\lbrace \tmu_t \rbrace_{t=1}^{T}$ are then different from $\lbrace \mu_t\rbrace_{t=1}^{T}$ (i.e., the ones that correspond to the true transition model), and hence the constraint \Cref{eqn:safety} guarantees only the safety under the estimated transitions $\tf$, i.e., $h_C(\tmu_t) {}\geq{} 0$. In contrast, the original environment constraint $h_C(\mu_t) {}\geq{} 0$ might be violated, resulting in unsafe mean-field distributions under true transitions $f$. 

Next, we demonstrate how to modify the constraint \Cref{eqn:safe-mf-rl-safety} for the optimization problem \Cref{eqn:safe-mf-rl} when an estimated transition function $\tf$ is used from the confidence set $\F_{n-1}$ such that the mean-field distributions $\mu_{n,t}$ induced by the resulting policy profile $\bpi_n^*$ do not violate the original constraint under true transitions $f$. First, we require the following property for any safety function $h(\cdot)$.

\begin{assumption}[Safety Lipschitz continuity] \label{asm:safety_lipschitz}
The safety function $h(\cdot)$ is $L_h$-Lipschitz-continuous, i.e., $|h(\mu) - h(\mu')| {}\leq{} L_h W_1(\mu, \mu').$ \looseness=-1
\end{assumption}
The following lemma shows that we can ensure safety under true transitions $f$ by having tighter constraints under any estimated transitions $\tf$ selected from $\F_{n-1}$. \looseness=-1

\begin{lemma} \label{lemma:hallucinated_safety}
Given a fixed policy profile $\bpi_n$, a safety function $h(\cdot)$ satisfying \Cref{asm:safety_lipschitz} and a safety threshold $C \in \R$, we have in episode $n$ for all steps $t$ \looseness=-1
\begin{equation*}\label{eqn:gap}
    |h_C(\tmu_{n,t}) - h_C(\mu_{n,t})| {}{}\leq{}{} L_h C_{n,t},
\end{equation*}
where $C_{n,t}$ is an arbitrary constant that satisfies $C_{n,t} {}\geq{} W_1(\tilde{\mu}_{n,t}, \mu_{n,t})$.  
\end{lemma}
\begin{proof}
    For arbitrary $\tmu_{n,t}, \mu_{n,t} \in \pp(\s)$ we have $$|h_C(\tmu_{n,t}) {}-{} h_C(\mu_{n,t})| {}={} |h(\tmu_{n,t}) {}-{} h(\mu_{n,t})| {}{}\leq{}{} L_h W_1(\tmu_{n,t}, \mu_{n,t}) {}{}\leq{}{} L_hC_{n, t},$$ where the first equality follows from the definition of $h_C(\cdot)$, the first inequality follows from \Cref{asm:safety_lipschitz} and the second inequality comes from $C_{n,t} {}\geq{} W_1(\tilde{\mu}_{n,t}, \mu_{n,t})$.\looseness=-1
\end{proof}

Crucially, using \Cref{lemma:hallucinated_safety} we can formulate a safety constraint for the optimization under estimated transitions $\tf$ that ensures that the constraint under true transitions $f$ is satisfied with high probability. \looseness=-1

\begin{corollary} \label{cor:safety_guarantee}
    For every episode $n$ and step $t$, $h_C(\tmu_{n,t}) {}\geq{} L_h C_{n,t}$ implies $h_C(\mu_{n,t}) {}\geq{} 0$ guaranteeing the safety of the original system.
\end{corollary}
\begin{proof}
    The corollary follows directly from \Cref{lemma:hallucinated_safety} and the triangle inequality, which are used in the third and the second inequality, respectively
    \begin{align*}
        L_h C_{n,t} 
        &{}\leq{} h_C(\tmu_{n,t}) \\ 
        &{}\leq{} |h_C(\tmu_{n,t}) - h_C(\mu_{n,t})| {}+{} h_C(\mu_{n,t}) \\
        &{}\leq{} L_h C_{n,t} {}+{} h_C(\mu_{n,t}).
    \end{align*}
    The claim is obtained by subtracting the positive constant $L_h C_{n,t}$ from both sides.
\end{proof} 

Then, $C_{n,t}$ for $t=1,\dots,T$ become parameters of the optimization problem (as defined in \Cref{sec:proposed_method}) that the representative agent faces at the beginning of episode $n$. However, choosing the appropriate values that comply with the condition $C_{n,t} {}\geq{} W_1(\tilde{\mu}_{n,t}, \mu_{n,t})$ is not trivial since $\mu_{n,t}$ depends on unknown true transitions of the system. Note that computing $C_{n,0}$ at the initial step $t=0$ is not necessary because the inequality is always guaranteed due to the initialization $\tmu_{n,0}=\mu_{n,0}$ for every episode $n$. In \Cref{apx:gp_upper_bound}, we demonstrate how to efficiently upper bound $W_1(\tilde{\mu}_{n,t}, \mu_{n,t})$ and obtain $C_{n,t}$ using the Lipschitz constants of the system and the statistical model's epistemic uncertainty. In particular, $C_{n,t}$ approaches zero, and $h_C(\tmu_{n,t}) {}\geq{} L_h C_{n,t}$ reduces to the constraint \Cref{eqn:safe-mf-rl-safety} as the estimated confidence $\bsigma_{n-1}(\cdot)$ shrinks due to the increasing number of observations available to estimate true transitions. \looseness=-1

\begin{algorithm*}[hbt!]
  \caption{Model-Based Learning Protocol in \algnme}
  \begin{algorithmic}[1]
    \Require Set of admissible policy profiles $\Pi$, safety constraint $h_C(\cdot)$, calibrated statistical model represented by $\bM_{n-1}(\cdot)$ and $\bSigma_{n-1}(\cdot)$, initial mean-field distribution $\mu_0$, known reward $r(\cdot)$, safety Lipschitz constant $L_h$, hyperparameter $\lambda$, number of episodes $N$, number of steps $T$ 
    \For{$n {}={}1,\ldots N$}
        \State Compute $C_{n,t}$ for $t {}={}1,\ldots,T$ as described in \Cref{apx:gp_upper_bound}
        \State Optimize the objective in \Cref{eqn:calibrated_safe} over the admissible policy profiles $\Pi$ and hallucinated transitions \Cref{eqn:hallucinated_transitions}
        \State Execute the obtained policy profile $\bpi_n^*$ and collect the trajectories $\D_n=\{(z_{n,t},s_{n,t{}+{}1})\}_{t=0}^{T-1}$ from the representative agent
        \State Update the confidence set of transitions $\F_{n-1}$ with the collected data to obtain $\F_{n}$ for the next episode
    \EndFor
    \Ensure $\bpi_N^* {}={}(\pi_{N,0}^*,\ldots,\pi_{N,T-1}^*)$
  \end{algorithmic}
  \label{alg:learning_protocol}
\end{algorithm*}

\section{\algnme} \label{sec:proposed_method}
In this section, we introduce a model-based approach for the \textit{Safe Mean-Field Reinforcement Learning} problem that combines the safety guarantees in \Cref{cor:safety_guarantee} with upper-bound confidence interval optimization. At the beginning of each episode $n$, the representative agent constructs the confidence set of transitions $\F_{n-1}$ (see \Cref{eq:model_class}) given the calibrated statistical model and previously observed data and selects a safe \textit{optimistic} policy profile $\bpi_n^*$ to obtain the highest value function within $\F_{n-1}$ while satisfying the safety constraint derived in \Cref{cor:safety_guarantee}.
In particular, the optimal policy profile $\bpi^*$ from \Cref{eqn:safe-mf-rl} is approximated at the episode $n$ by \looseness=-1
\begin{subequations}\label{eqn:calibrated_short}
\begin{align} 
    \bpi_n^* {}={}\argmax_{\bpi_n \in \Pi} {}& \max_{\tf_{n-1} \in \F_{n-1}} \E \Bigg [ \sum_{t=0}^{T-1} r(\tz_{n,t}) \Big\rvert \tmu_{n,0} {}={}\mu_{0}\Bigg ] \\
    \text{subject to} \quad \ta_{n,t} &{}={}\pi_{n,t}(\ts_{n,t}, \tmu_{n,t}) \\
    \ts_{n,t{}+{}1} &{}={}\tf_{n-1}(\tz_{n,t}) {}+{} \epsilon_{n,t} \\
    \tmu_{n,t{}+{}1} &{}={}U(\tmu_{n,t}, \pi_{n,t}, \tf_{n-1}) \\
    h_C(\tmu_{n,t{}+{}1}) &{}\geq{} L_h C_{n,t{}+{}1},
    \label{eqn:hard_safety_constraint}
\end{align}
\end{subequations}
with $\tz_{n,t} {}={}(\ts_{n,t}, \tmu_{n,t}, \ta_{n,t})$. \Cref{eqn:calibrated_short} optimizes over the function space $\F_{n-1}$ which is usually intractable even in bandit settings~\cite{dani2008stochastic}. Additionally, it must comply with the safety constraint \Cref{eqn:hard_safety_constraint}, further complicating the optimization. We utilize the \textit{hallucinated control} reparametrization and the \textit{log-barrier} method to alleviate these issues. After the reformulation of the problem, model-free or model-based mean-field optimization algorithms can be applied to find policy profile $\bpi_n^*$ at the beginning of episode $n$. \looseness=-1

We use an established approach known as \textit{Hallucinated Upper Confidence Reinforcement Learning} (H-UCRL)~\cite{moldovan2015optimism, curi2020efficient, pasztor2021efficient} and introduce an auxiliary function $\eta:\zz \rightarrow [-1, 1]^p$, where $p$ is the dimensionality of the state space $\s$, to define hallucinated transitions \looseness=-1
\begin{equation}\label{eqn:hallucinated_transitions}
    \tf_{n-1}(z) {}={}\bM_{n-1}(z) {}+{} \beta_{n-1}\bSigma_{n-1}(z)\eta(z).
\end{equation}
Notice that $\tf_{n-1}$ is calibrated for any $\eta(\cdot)$ under \Cref{asm:calibrated_model}, i.e., $\tf_{n-1} \in \F_{n-1}$. \Cref{asm:calibrated_model} further guarantees that every function $\tf_{n-1}$ can be expressed in the auxiliary form \Cref{eqn:hallucinated_transitions}
\begin{equation*}
\begin{split}
    &\forall \tf_{n-1} \in \F_{n-1} \; \exists \eta:\zz \rightarrow [-1, 1]^p \; \text{such that} \\
    {}&\tf_{n-1}(z) {}={}\bM_{n-1}(z) {}+{} \beta_{n-1}\bSigma_{n-1}(z)\eta(z), \; \forall z \in \zz.
\end{split}
\end{equation*}
Thus, the intractable optimization over the function space $\F_{n-1}$ in \Cref{eqn:calibrated_short} can be expressed as an optimization over the set of admissible policy profiles $\Pi$ and auxiliary function $\eta(\cdot)$ (see \Cref{apx:hallucinated_control} for further details).
Note that $\eta(z) {}={}\eta(s, \mu, \pi(s, \mu)) {}={}\eta(s, \mu)$ for a fixed individual policy $\pi$. This turns $\eta(\cdot)$ into a policy that exerts \textit{hallucinated control} over the epistemic uncertainty of the confidence set of transitions $\F_{n-1}$~\cite{curi2020efficient}. Furthermore, \Cref{eqn:hallucinated_transitions} allows us to optimize over parametrizable functions (e.g., \textit{neural networks}) $\bpi$ and $\eta(\cdot)$ using gradient ascent.

We introduce the safety constraint to the objective using the \textit{log-barrier method}~\cite{wright1992interior}. This restricts the domain on which the objective function is defined only to values that satisfy the constraint \Cref{eqn:hard_safety_constraint}, hence, turning \Cref{eqn:calibrated_short} to an unconstrained optimization problem. Combining these two methods yields the following optimization problem \looseness=-1
\begin{subequations} \label{eqn:calibrated_safe}
\begin{align} 
    \begin{split}
       \bpi_n^* &{}={}\argmax_{\bpi_n \in \Pi} \max_{\eta(\cdot) \in [-1,1]^p} \\   
        \E \Bigg[ \sum_{t=0}^{T-1} r(\tz_{n,t}) &{}+{} \lambda\log\left(h_C(\tmu_{n,t{}+{}1}) - L_h C_{n,t{}+{}1}\right) \Big\rvert \tmu_{n,0} {}={}\mu_{0} \Bigg] 
    \end{split} \label{eqn:log-barrier-objective} \\
    \text{subject to} \quad \ta_{n,t} &{}={}\pi_{n,t}(\ts_{n,t}, \tmu_{n,t}) \\
    \tf_{n-1}(\tz_{n,t}) &{}={}\bM_{n-1}(\tz_{n,t}) {}+{} \beta_{n-1}\bSigma_{n-1}(\tz_{n,t})\eta(\tz_{n,t}) \\
    \ts_{n,t{}+{}1} &{}={}\tf_{n-1}(\tz_{n,t}) {}+{} \epsilon_{n,t} \\
    \tmu_{n,t{}+{}1} &{}={}U(\tmu_{n,t}, \pi_{n,t}, \tf_{n-1}),
\end{align}
\end{subequations}
with $\tz_{n,t} {}={}(\ts_{n,t}, \tmu_{n,t}, \ta_{n,t})$ and $\lambda > 0$ being a tuneable hyperparameter used to balance between the reward and the safety constraint. Provided that the set of safe mean-field distributions (assuming the safe initial distribution $\mu_0$ is given) is not empty, $\bpi_n^*$ is guaranteed to satisfy the safety constraint during the policy rollout in episode $n$.\looseness=-1
\begin{remark}
    Note that \Cref{eqn:calibrated_safe} can also be used under known transitions by setting $\bM_{n-1}(\cdot)=f(\cdot)$, $\bSigma_{n-1}(\cdot) {}={}0$ and $L_hC_{n,t}=0$, hence, recovering the original constraint $h_C(\tmu_{n,t}) {}\geq{} 0$ from \Cref{eqn:safety}. In \Cref{sec:experiments}, we utilize this useful property to construct the upper bound for the reward obtained under unknown transitions.
\end{remark}
We summarize the model-based learning protocol used by \algnmespace in \Cref{alg:learning_protocol}. 
The first step computes constants $C_{n,t}$ (see \Cref{apx:gp_upper_bound}) introduced in \Cref{lemma:hallucinated_safety}. The second step optimizes the objective in \Cref{eqn:calibrated_safe}. The third and fourth steps collect trajectories from the representative agent and update the calibrated model. 
While the learning protocol is model-based, the subroutine in Line 3 can use either model-based or model-free algorithms proposed for the MFC due to our reformulation in \Cref{eqn:calibrated_safe}. In \Cref{apx:mf-algorithms}, we introduce modifications of well-known algorithms for optimizing the mean-field setting. \looseness=-1

\section{Experiments} \label{sec:experiments}
In this section, we demonstrate the performance of \algnmespace \newline on the swarm motion benchmark and showcase that it can tackle the real-world large-scale vehicle repositioning problem faced by ride-hailing platforms.\looseness=-1

\begin{figure*}[!hbt!]
\centering
\begin{subfigure}[t]{0.32\textwidth}
    \centering
    \includegraphics[width=\textwidth]{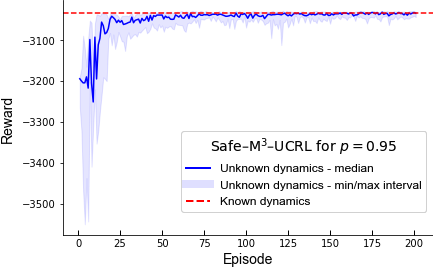}
    \captionsetup{justification=centering}
    \vspace{-0.5cm}
    \caption{Swarm motion learning curve}
    \label{fig:swarm_learning_curve}
\end{subfigure}
\hfill
\begin{subfigure}[t]{0.32\textwidth}
    \centering
    \includegraphics[width=\textwidth]{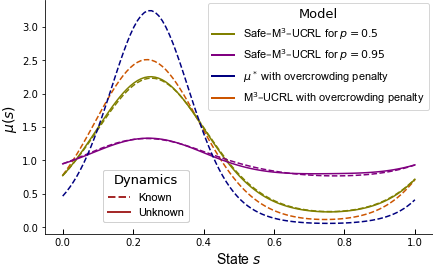}
    \captionsetup{justification=centering}
    \vspace{-0.5cm}
    \caption{Swarm motion mean-field distributions}
    \label{fig:swarm_distributions}
\end{subfigure}
\hfill
\begin{subfigure}[t]{0.32\textwidth}
    \centering
    \includegraphics[width=\textwidth]{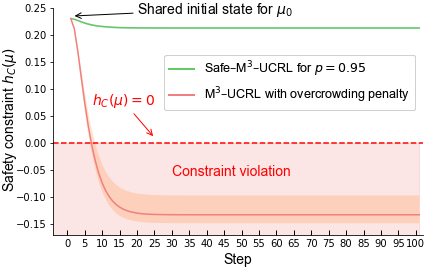}
    \captionsetup{justification=centering}
    \vspace{-0.5cm}
    \caption{Swarm motion safety for $p = 0.95$}
    \label{fig:swarm_constraint_violations}
\end{subfigure}

\begin{subfigure}[t]{0.32\textwidth}
    \centering
    \includegraphics[width=\textwidth]{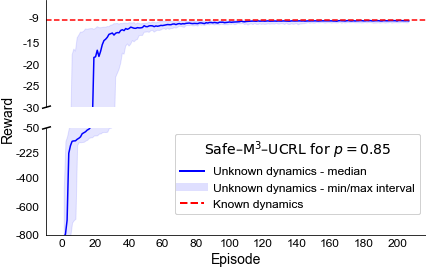}
    \captionsetup{justification=centering}
    \vspace{-0.5cm}
    \caption{Vehicle repositioning learning curve}
    \label{fig:repositioning_learning_curve}
\end{subfigure}
\hfill
\begin{subfigure}[t]{0.32\textwidth}
    \centering
    \includegraphics[width=\textwidth]{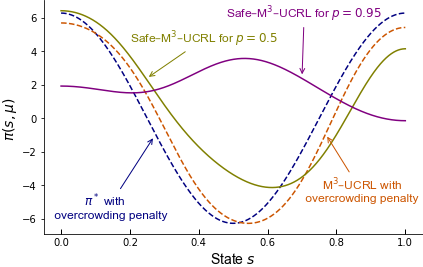}
    \captionsetup{justification=centering}
    \vspace{-0.5cm}
    \caption{Swarm motion policies}
    \label{fig:swarm_policies}
\end{subfigure}
\hfill
\begin{subfigure}[t]{0.32\textwidth}
    \centering
    \includegraphics[width=\textwidth]{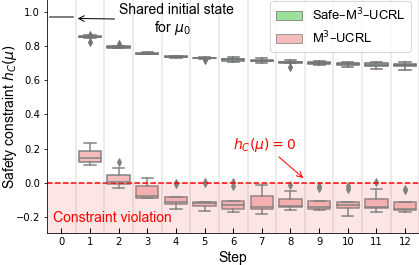}
    \captionsetup{justification=centering}
    \vspace{-0.5cm}
    \caption{Vehicle repositioning safety for $p = 0.85$}
    \label{fig:repositioning_constraint_violations}
\end{subfigure}
\vspace{-0.3cm}
\caption{Performance analysis of \algnmespace for swarm motion and vehicle repositioning. The policy and statistical model were trained on 10 randomly initialized neural networks for each hyperparameter $p$.}
\label{fig:performance_analysis}
\Description{Performance analysis of the proposed algorithm for swarm motion and vehicle repositioning. The policies and statistical model were trained on 10 randomly initialized neural networks for each hyperparameter $p$ representing the proportion of maximum entropy used as a constraint. The figure displays learning curves, mean-field distributions, safety violation graphs, and policies.}
\end{figure*}

\subsection{Swarm Motion} \label{sec:swarm_motion}
Due to the infancy of Mean-Field RL as a research topic, one of the rare benchmarks used by multiple authors is the swarm motion. \cite{pasztor2021efficient,carmona2019model} view it as Mean-Field RL problem, while~\cite{elie2020convergence} uses it in the context of MFGs. In this setting, an infinite population of agents is moving around toroidal state space with the aim of maximizing a location-dependent reward function while avoiding congested areas~\cite{almulla2017two}. \looseness=-1

\textbf{Modeling.} We model the state space $\s$ as the unit torus on the interval $[0,1]$, and the action space is the interval $\A {}={}[-7, 7]$. We approximate the continuous-time swarm motion by partitioning unit time into $T=100$ equal steps of length $\Delta t {}={}1{}/T$. The next state $s_{n,t{}+{}1} {}={}f(z_{n,t}) {}+{} \epsilon_{n,t}$ is induced by the unknown transitions $f(z_{n,t}) {}={}s_{n,t} {}+{} a_{n,t}\Delta t$ with $\epsilon_{n,t} \sim \text{N}(0, \Delta t)$ for all episodes $n$ and steps $t$. The reward function is defined by $r(z_{n,t}) {}={}\phi(s_{n,t}) - \frac{1}{2}a_{n,t}^2 - \log(\mu_{n,t})$, where the first term $\phi(s) {}={}2 \pi^2 (\sin(2\pi s) - \cos^2(2\pi s)) {}+{} 2\sin(2\pi s)$ determines the positional reward received at the state $s$ (see \Cref{apx:swarm_evaluation}), the second term defines the kinetic energy penalizing large actions, and the last term penalizes overcrowding. Note that the optimal solution for continuous time setting, $\Delta t \rightarrow 0$ can be obtained analytically~\cite{almulla2017two} (see \Cref{apx:swarm_modeling}) and used as a benchmark. \cite{pasztor2021efficient,carmona2019model} show that Mean-Field RL discrete-time, $\Delta t$ > 0, methods can learn good approximations of the optimal solution. The disadvantage of these methods is that they can influence the skewness of the mean-field distribution only via overcrowding penalty. Therefore, to control skewness, their only option is to introduce a hyperparameter to the reward that regulates the level of overcrowding penalization. On the other hand, \algnmespace controls skewness without trial-and-error reward shaping by imposing the entropic safety constraint $h_C(\mu_{n,t}) {}={}H(\mu) - C \geq 0$, with $H(\cdot)$ defined in \Cref{eqn:differential_entropy}, instead of having the overcrowding penalty term $\log(\mu_{n,t})$. Since higher entropy translates into less overcrowding, we can upfront determine and upper-bound the acceptable level of overcrowding by setting a desirable threshold $C$. 

We use a neural network to parametrize the policy profile $\bpi_n=(\pi_{n,0}, \ldots, \pi_{n,T-1})$, for every episode $n$, during the optimization of \Cref{eqn:calibrated_safe}. The optimization is done by \textit{Mean-Field Back-Propagation Through Time} (MF-BPTT) (see \Cref{apx:mf-bptt}). In our experiments, a single neural network shows enough predictive power to represent the whole policy profile, but using $T$ networks, one for each individual policy $\pi_{n,t}$, $t=0,\dots,T-1$, is a natural extension. We use a \textit{Probabilistic Neural Network Ensemble}~\cite{chua2018deep,lakshminarayanan2017simple} to represent a statistical model of transitions, which we elaborate in \Cref{apx:ensemble}.
We represent the mean-field distribution by discretizing the state space uniformly and assigning the probability of the representative agent residing within each interval.
We set the safety threshold $C$ as a proportion $p \in [0,1]$ of the maximum entropy, i.e., $C{}={}p\max_{\pp(\s)}H(\mu)$.
Note that \algnmespace guarantees safe mean-field distributions only if the initial mean-field distributions $\mu_{n,0}$ at time $t=0$ are safe for every episode $n$ for given threshold $C$. A generic way for safe initialization is setting $\mu_{n,0}$ as the maximum entropy distribution among all safe distributions $\zeta$ \looseness=-1
\begin{equation} \label{eq:mu_initialization}
    \mu_{n,0} = \argmax_{\mu\in\zeta}H(\mu).
\end{equation}
Note that, in general, the safe initial distribution might not exist.
We provide the implementation details in \Cref{apx:vehicle_modeling}. \looseness=-1

\textbf{Results.}
In \Cref{fig:swarm_learning_curve}, we observe the learning curve of \algnmespace for $p=0.95$ for 10 randomly initialized runs. The learning process is volatile in the initial phase due to the high epistemic uncertainty, but after 50 episodes, all policies converge toward the solution as if the transitions were known. In \Cref{fig:swarm_distributions}, we use various thresholds $C$ to show that the entropic constraint effectively influences the degree of agents' greediness to collect the highest positional reward. By increasing $p$ towards 1, we force agents to put increasingly more emphasis on global welfare rather than on individual rewards. We see that for $p=0.5$ we obtain a distribution that matches the distribution obtained by the unconstrained \ucrlspace~\cite{pasztor2021efficient} that relies on the overcrowding penalty, while for $p=0.95$ we significantly surpass the effect that the penalty has on the distribution's skewness. We also show that the discrete-time solutions with low $p$ serve as a good approximation of the continuous-time optimal distribution $\mu^*$. Furthermore, we observe that the policies under unknown transitions overlap with the solutions learned under known transitions. In \Cref{fig:swarm_policies}, we observe that unconstrained policies and policies with low $p$ push agents towards high individual rewards. Note that for states close to 1, the algorithms learn it requires less kinetic energy to push agents over the border due to the toroidal shape of the state space. For high $p$'s learned policies push agents always in the same direction to maintain uniformity of the system. Importantly, \Cref{fig:swarm_constraint_violations} shows that \algnmespace for $p=0.95$ keeps the mean-field distribution safe throughout the entire execution, unlike the solutions that rely on the overcrowding penalty term. For further details, see \Cref{apx:swarm_evaluation}. \looseness=-1

\begin{figure*}[hbt!]
\centering
\begin{subfigure}[t]{0.232\textwidth}
    \centering
    \includegraphics[width=\textwidth]{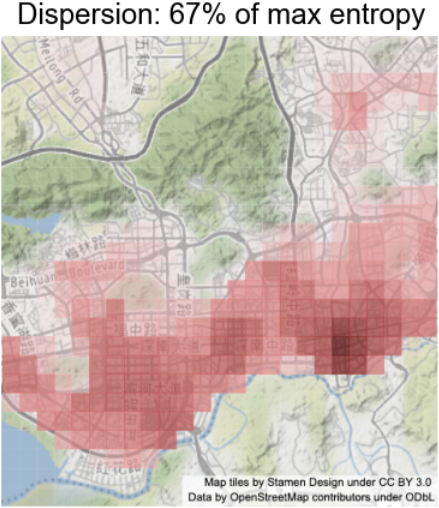}
    \captionsetup{justification=centering}
    \caption{\\Observed demand $\rho_0$ used as a target distribution}
    \label{fig:target_mu}
\end{subfigure}
\hfill
\begin{subfigure}[t]{0.232\textwidth}
    \centering
    \includegraphics[width=\textwidth]{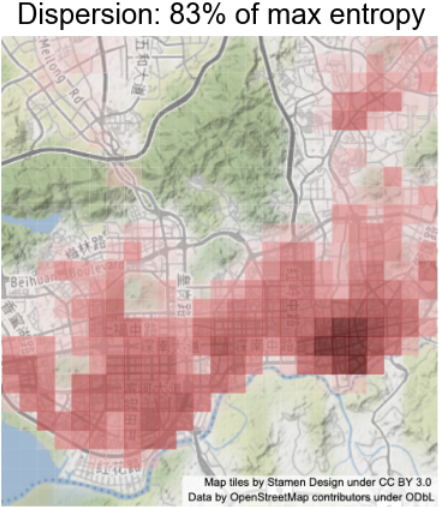}
    \captionsetup{justification=centering}
    \caption{\\Unconstrained \ucrlspace under known transitions}
    \label{fig:ucrl_mu}
\end{subfigure}
\hfill
\begin{subfigure}[t]{0.232\textwidth}
    \centering
    \includegraphics[width=\textwidth]{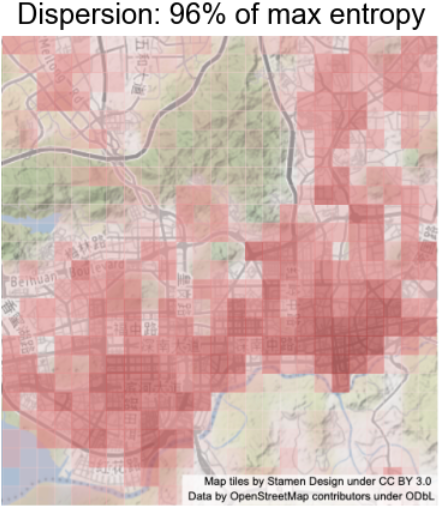}
    \captionsetup{justification=centering}
    \caption{\\\algnmespace under known transitions for $p=0.85$}
    \label{fig:safe_ucrl_mu_known}
\end{subfigure}
\hfill
\begin{subfigure}[t]{0.232\textwidth}
    \centering
    \includegraphics[width=\textwidth]{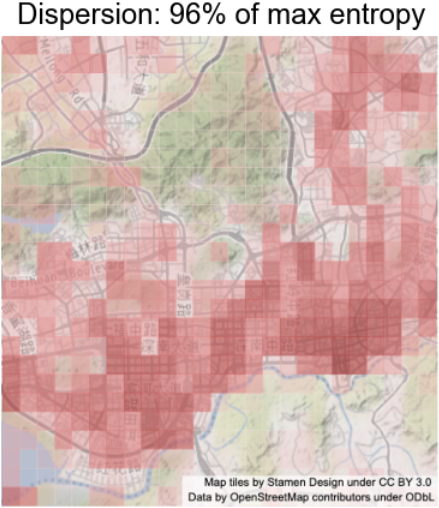}
    \captionsetup{justification=centering} 
    \caption{\\\algnmespace under unknown transitions for $p=0.85$}
    \label{fig:safe_ucrl_mu_unknown}
\end{subfigure}
\hfill
\begin{subfigure}[b]{0.05\textwidth}
    \centering
    \includegraphics[height=4.7cm]{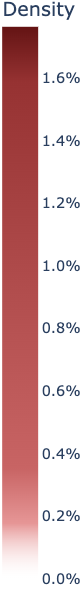}
\end{subfigure}
\vspace{-0.2cm}
\caption{\algnmespace guided vehicle distribution in Shenzhen in the evening peak hours.} \label{fig:repositioning_mu_comparison}
\Description{The proposed algorithm guided vehicle distribution in Shenzhen in the evening peak hours. The image shows heatmaps illustrating the spatial distribution of vehicles.}
\end{figure*} 

\subsection{Vehicle Repositioning Problem} \label{sec:vehicle_repositioning}
Since ride-hailing services, such as Uber, Lyft, and Bolt, gained popularity and market share, vehicle repositioning has been a long-standing challenge for these platforms, i.e., moving idle vehicles to areas with high-demand potential. A similar challenge is present in bike-sharing services accessible in many cities and, as of more recently, dockless electric scooter-sharing services such as Bird and Lime.
In the competitive environment, the operator significantly increases profit by successfully repositioning idle vehicles to high-demand areas. Nevertheless, there might exist regulations imposed by the countries or cities that enforce service providers to either guarantee fair service accessibility or restrict the number of vehicles in districts with high traffic density. Such restrictions prevent operators from greedily maximizing the profit and can be encapsulated by some dispersion metric such as entropy. 
Solving this problem helps prevent prolonged vehicle cruising and extensive passenger waiting times in the demand hotspots, increasing the service provider's efficiency and reducing its carbon footprint.
Existing approaches to vehicle repositioning range from static optimization over a queuing network~\cite{zhang2016control,braverman2019empty}, model predictive control~\cite{iglesias2018data}, to RL~\cite{wen2017rebalancing,lin2018efficient,mao2020dispatch}. 
The main advantage of \algnmespace is the capability of controlling a \textit{large} fleet of homogeneous vehicles and enforcing efficient coordination to match the spatiotemporal demand distribution. Additionally, the safety constraint introduced into the model guarantees service accessibility by ensuring idle vehicles are spreading over the study region. Although accessibility has not been widely discussed in the literature on vehicle repositioning, it is expected to be an important fairness constraint when shared mobility services become a prevailing travel mode~\cite{shaheen2017travel}. \looseness=-1

\textbf{Modeling.} 
Ride-hailing operations can be modeled as sequential decision-making, which consists of passenger trips followed by repositioning trips operated by a central controller as illustrated in \Cref{fig:trajectory}. We assume that the controller has access to the locations of vehicles in its fleet and communicates the real-time repositioning actions to the drivers via electronic devices. Nevertheless, since the fleet is operating in a noisy traffic environment, repositioning usually cannot be executed perfectly. We assume that vehicles can move freely within the area of our interest, which is represented by a two-dimensional unit square, i.e., the state-space $\s {}={}[0,1]^2$, and repositioning actions are taken from $\A {}={}[-1,1]^2$.
The objective of our model is to satisfy the demand in the central district of Shenzhen while providing service accessibility in the wider city center. We restrict our modeling horizon to three evening peak hours, which are discretized in fifteen-minute operational intervals, i.e., $T=12$, and each episode $n$ represents one day (for more details, see \Cref{apx:vehicle_repositioning}).
We model service providers goal of maximizing the coverage of the demand by the negative of the Kullback-Leibler divergence between the vehicles' distribution $\mu_{n,t}$ and demand for service denoted as $\rho_0$, i.e., $r(z_{n,t}) {}={}-D_{KL}(\rho_0 || \mu_{n,t})$. 
In particular, the demand distribution $\rho_0 \in \pp(\s)$ represents a probability of a trip originating in the infinitesimal neighborhood of state $s\in\s$ during peak hours (see \Cref{fig:target_mu}). We estimate a stationary demand distribution $\rho_0$ from the vehicle trajectories collected in Shenzhen, China, in 2016. 
If the passenger's trip originates at $s\in\s$ the likelihood of its destinations is defined by the mapping $\Phi:\s\rightarrow\pp(\s)$, which we fit from the trip trajectories (see \Cref{apx:vehicle_preprocessing}). We use $\Phi(\cdot)$ to define sequential transitions by first executing passenger trips followed by vehicle repositioning. Formally, the next state $s_{n,t{}+{}1} {}={}f(z_{n,t}) {}+{} \epsilon_{n,t}$ is induced by the unknown transitions $f(z_{n,t}) {}={}\text{clip}(s^\Phi_{n,t} {}+{} a_{n,t}, 0, 1)$, where $s^\Phi_{n,t}\sim\Phi(s_{n,t})$ and $\epsilon_{n,t}\sim \text{TN}(0, \sigma^2 I_2)$ is a Gaussian with a known variance $\sigma^2$ truncated at the borders of $\s$ and $I_2$ is the $2\times 2$ unit matrix. Notice that the controller determines repositioning actions given intermediate states $s^\Phi_{n,t}$ obtained after executing passenger trips. We use entropic safety constraint $h_C(\mu_{n,t}) {}={}H(\mu) - C \geq 0$ to enforce the service accessibility across all residential areas (see \Crefrange{fig:safe_ucrl_mu_known}{fig:safe_ucrl_mu_unknown}). Therefore, the optimization objective in \Cref{eqn:calibrated_safe} trades off between greedily satisfying the demand $\rho_0$ and adhering to accessibility constraint imposed by $h_C(\cdot)$. Identically to the swarm motion experiment, we use a neural network to parametrize the policy profile $\bpi_n$, which we optimize by MF-BPTT. A statistical model of the transitions is represented by a Probabilistic Neural Network Ensemble, while $\mu_{n,0}$ is initialized using \Cref{eq:mu_initialization}. We represent the mean-field distribution by discretizing the state space into the uniform grid as elaborated in \Cref{apx:vehicle_preprocessing}.

\textbf{Results.} 
The entropy of the target distribution, $\rho_0$ in \Cref{fig:target_mu}, already achieves $p=0.67$ of the maximum due to a wide horizontal spread. To achieve vertical spread, we require an additional 18 percentage points of entropy as a safety constraint. Concretely, we use $p=0.85$ to set the threshold as the proportion of maximum entropy and proceed by optimizing the policy profile in \Cref{eqn:calibrated_safe}. Due to the lack of an analytical solution for \Cref{eqn:calibrated_safe}, we use a policy profile trained under known transitions as a benchmark. We observe that the learned policy profile $\bpi_n^*$ converges to the policy profile under known transitions in $n=80$ episodes. \Cref{fig:repositioning_learning_curve} shows two phases of the learning process. During the first 60 episodes, the performance is volatile, but once the epistemic uncertainty around true transitions is tight, the model exploits it rapidly by episode 80. 
\looseness=-1

In \Cref{fig:repositioning_constraint_violations}, we empirically show that \algnmespace satisfies safety constraints during the entire execution. In \Cref{fig:repositioning_mu_comparison}, we use a city map of Shenzhen to show that \algnmespace improves service accessibility in low-demand areas. \Cref{fig:ucrl_mu} shows that \ucrlspace under known transitions learns how to satisfy the demand $\rho_0$ effectively at the cost of violating safety constraint (see \Cref{fig:repositioning_constraint_violations}). \Cref{fig:safe_ucrl_mu_known} shows that \algnmespace under known transitions improves safety by distributing vehicles to residential areas in the northwest and northeast.
Finally, \Cref{fig:safe_ucrl_mu_unknown} emphasizes the capability of \algnmespace to learn complex transitions while the policy profile $\bpi_n^*$ simultaneously converges towards the results achieved under known transitions with the number of episodes $n$ passed.\looseness=-1

In \Cref{apx:vehicle_learning_protocol}, we provide the details on the parameters used during the training and exhaustive performance analysis. The results in this section are generated assuming the infinite regime. At the same time, in \Cref{apx:vehicle_evaluation_finite_regime}, we showcase that in the finite regime the policy profile $\bpi_n^*$ can be successfully applied to millions of individual agents in real-time, which might be of particular importance to real-world practitioners. The code we use to train and evaluate \algnmespace is available in our GitHub repository~\cite{Jusup_Safe_Model-Based_Multi-Agent_2023}. 

\section{Conclusion} \label{sec:conclusion} 
We present a novel formulation of the mean-field model-based reinforcement learning problem incorporating safety constraints. \algnmespace addresses this problem by leveraging epistemic uncertainty under an unknown transition model and employing a log-barrier approach to ensure conservative satisfaction of the constraints. Beyond the synthetic swarm motion experiment, we showcase the potential of our algorithm for real-world applications by effectively matching the demand distribution in a shared mobility service while consistently upholding service accessibility. In the future, we believe that integrating safety considerations in intelligent multi-agent systems will have a crucial impact on various applications, such as autonomous ride-hailing, firefighting robots and drone/robot search-and-rescue operations in complex and confined spaces.

\balance



\begin{acks}
The authors would like to thank Mojm\'ir Mutn\'y for the fruitful discussions during the course of this work. Matej Jusup acknowledges support from the Swiss National Science Foundation under the research project DADA/181210. This publication was made possible by an ETH AI Center doctoral fellowship to Barna Pasztor. Francesco Corman acknowledges Grant 2023-FS-331 for Research in the area of Public Transport. Andreas Krause acknowledges that this research was supported by the European Research Council (ERC) under the European Union’s Horizon 2020 research and innovation program grant agreement no. 815943 and the Swiss National Science Foundation under NCCR Automation, grant agreement 51NF40 180545. Ilija Bogunovic acknowledges that the work was supported in part by the EPSRC New Investigator Award EP/X03917X/1.
The authors would also like to thank the Shenzhen Urban Transport Planning Center for collecting and sharing the data used in this work. \looseness=-1
\end{acks}



\bibliographystyle{ACM-Reference-Format} 
\bibliography{bibliography}


\appendix

\twocolumn[
  \begin{@twocolumnfalse}
    \centering
    {\fontsize{32}{36}\selectfont Appendix} \\
    \vspace{0.5cm}
    {\fontsize{18}{22}\selectfont Safe Model-Based Multi-Agent \\ Mean-Field Reinforcement Learning} \\
    \vspace{1cm}
  \end{@twocolumnfalse}
]

\section{Relationship Between Mean-Field Distributions under True and Estimated Transitions} \label{apx:gp_upper_bound}

In this section, we describe the procedure for computing constant $C_{n,t}$ (defined in \Cref{sec:proposed_method}) at every episode $n {}={}1,\ldots,N$ and step $t {}={}1,\ldots,T$. This is necessary for establishing the connection between the mean-field distribution $\mu_{n,t}$ in the original system under known transitions $f$ (see \Cref{eqn:safe-mf-rl}) and the mean-field distribution $\tmu_{n,t}$ in the system induced by a calibrated statistical model $\tf_{n-1}$ (see \Cref{eqn:calibrated_short}, \Cref{sec:statistical_model} and \Cref{asm:calibrated_model}). In particular, \Cref{cor:safety_guarantee} shows that the safety in the original system $h_C(\mu_{n,t})$ is guaranteed with high probability if $h_C(\tmu_{n,t}) {}\geq{} L_h C_{n,t}$ for a safety constant $C_{n,t} {}\geq{} W_1(\tmu_{n,t}, \mu_{n,t})$ and Lipschitz constant $L_h$ introduced in \Cref{asm:safety_lipschitz}. 
We use the following result from~\cite[Lemma 5]{pasztor2021efficient} to make a connection between the mean-field distributions under true and estimated transitions. \looseness=-1
\begin{lemma}
\label{lemma:wasserstein}
Under \Crefrange{asm:transition_lipschitz}{asm:calibrated_lipschitz} and assuming that the event in \Cref{asm:calibrated_model} holds true, for episodes $n {}={}1,\ldots,N$, steps $t {}={}1,\ldots,T$  and fixed policy profile $\bpi_n {}={}(\pi_{n,0}, \ldots, \pi_{n,T-1})$, we have:
\begin{equation*} \label{eqn:uncertainty}
W_1(\tmu_{n,t}, \mu_{n,t}) {}\leq{} 2 \beta_{n-1} \overline{L}_{n-1}^{t-1} \sum_{i=0}^{t-1} I_{n,t},
\end{equation*}
where $I_{n,t} {}={}\int_{\s} \| \bsigma_{n-1}(s,\mu_{n, i}, \pi_{n,t}(s, \mu_{n,i}))\|_2 \mu_{n,i}(ds)$ and $\overline{L}_{n-1} {}={}1 {}+{} 2(1 {}+{} L_{\pi})(L_f {}+{} 2\beta_{n-1}L_{\sigma})$.
\end{lemma}
Let $K_{n,t}{}:={}2\beta_{n-1} \overline{L}_{n-1}^{t-1}$ and $z {}={}(s, \mu, \pi_{n,t}(s, \mu))$. Then we have:
\begin{align*}
    W_1(\tmu_{n,t}, \mu_{n,t}) 
    &{}\leq{} K_{n,t}\sum_{i=0}^{t-1} \int_{\s} \| \bsigma_{n-1}(s,\mu_{n, i}, \pi_{n,t}(s, \mu_{n,i}))\|_2 \mu_{n,i}(ds) \\ 
    &{}\leq{} tK_{n,t}\max_{z\in \zz} \| \bsigma_{n-1}(z) \|_2, 
\end{align*}
for $t=1,\ldots,T$, where the first inequality is due to \Cref{lemma:wasserstein} while the second one follows since maximum upper bounds expectation. Hence, we can set $C_{n,t}$ as $tK_{n,t}\max_{z \in\zz} \| \bsigma_{n-1}(z) \|_2$, and calculate $\max_{z \in\zz} \| \bsigma_{n-1}(z) \|_2$ once at the beginning of each episode $n$. Notice that our learning protocol (see \Cref{alg:learning_protocol}) does not require computing $C_{n,0}$ at the initial step $t=0$ since the mean-field distributions share the initial state, i.e., $\tmu_{n,0}=\mu_{n,0}$. It is worth noting that in certain models, such as Gaussian Processes, this upper bound provides a meaningful interpretation. Specifically, $\bsigma_{n-1}(z)$ represents the epistemic uncertainty of the model, which tends to decrease monotonically as more data is observed. \looseness=-1

\section{Examples of Safety Constraints}\label{apx:proof_useful_constraints}
In this section, we show some important classes of safety constraints $h_C(\cdot) \geq 0$ satisfying \Cref{asm:safety_lipschitz}. 

\subsection{Entropic Safety} \label{apx:entropic_safety}
Entropic safety serves the purpose of controlling the dispersion of the mean-field distribution, which is useful in many applications, such as vehicle repositioning (see \Cref{sec:experiments}). A natural way of defining entropic safety is via differential entropy, but in general, the differential entropy is not Lipschitz continuous due to the unboundedness of the natural logarithm. Nevertheless, the issue can be easily circumvented by considering $\epsilon$-smoothed differential entropy $H^\epsilon:\pp(\s)\rightarrow \R_{\geq 0}$. For $\epsilon > 0$ and $C \geq 0$ we define $\epsilon$-smoothed differential entropy and associated entropic safety constraint $H_C^\epsilon(\cdot)$ as \looseness=-1
\begin{align*}
    H^\epsilon(\mu) &{}:={}-\int_{\s}\log(\mu(s) {}+{} \epsilon)\mu(ds) \\
    H^\epsilon_C(\mu) &{}:={}H^\epsilon(\mu) - C
\end{align*}
To show that $H^\epsilon_C(\cdot)$ satisfies \Cref{asm:safety_lipschitz} let $h_C(\cdot) {}:={}H^\epsilon_C(\cdot)$ and assume that $\s\subset\R^p$ is a compact set. First, note that $f(x) {}={}\log(x{}+{}\epsilon)$ is $\frac{1}{\epsilon}$-Lipschitz continuous for $\epsilon > 0$, i.e., $\epsilon f(x)$ is 1-Lipschitz. Second, for every $S\subseteq\s$ and $L$-Lipschitz function $f:\s\rightarrow\R$, a function $g(S)=\int_{\s}f(s)\mu(ds)$ is $L$-Lipschitz due to the boundedness of $f$. The following derivation shows that $H_C^\epsilon(\cdot)$ is $\frac{1}{\epsilon}$-Lipschitz continuous.
\begin{align*}
    &|h_C(\mu) - h_C(\mu')| 
    {}={}\left| H^\epsilon_C(\mu) - H^\epsilon_C(\mu')
     \right| \\
    ={}& \left|
        \int_{\s} \log(\mu'(s) {}+{} \epsilon)\mu'(ds) - \int_{\s} \log(\mu(s) {}+{} \epsilon)\mu(ds)
    \right| \\
    \begin{split}
    ={}& \left|
        \int_{\s} \log(\mu'(s) {}+{} \epsilon)(\mu' -\mu)(ds) - \int_{\s} \log(\mu(s) {}+{} \epsilon)\mu(ds) \right. \\
        &\left. {}+{} \int_{\s} \log(\mu'(s){}+{}\epsilon)\mu(ds)
        \right|
    \end{split} \\
    \begin{split}
    {}\leq{}{}& \left|\int_{\s} \log(\mu'(s) {}+{} \epsilon)(\mu' -\mu)(ds)\right| \\
    &{}+{} \left| \int_{\s} \log(\mu(s) {}+{} \epsilon)\mu(ds) - \int_{\s} \log(\mu'(s){}+{}\epsilon)\mu(ds)\right|
    \end{split} \\
    {}\leq{}{}&  \left|
        \int_{\s} \log(\mu'(s) {}+{} \epsilon)(\mu' -\mu)(ds)
    \right| 
    {}+{} \frac{1}{\epsilon} W_1(\mu, \mu') \\ 
    {}\leq{}{}& M \left|
        \int_{\s} (\mu' -\mu)(ds)
    \right| {}+{} \frac{1}{\epsilon} W_1(\mu, \mu') \\
    ={}& \underbrace{M(\mu' -\mu)(\s)}_{\overline{M}} {}+{} \frac{1}{\epsilon} W_1(\mu, \mu') \\
    ={}& \overline{M} {}+{} \frac{1}{\epsilon} W_1(\mu, \mu'),
\end{align*}
where the first inequality comes from the triangle inequality, the second inequality comes from the Lipschitz continuity of the $\epsilon$-smoothed logarithm and integral, the third comes from the upper-boundedness of logarithm on a compact set, and the last equality comes from the fact that the measure of a compact set is finite.
\begin{remark} \label{rmk:shannon_entropy}
    In \Cref{sec:experiments}, \Cref{apx:vehicle_repositioning} and \Cref{apx:swarm_motion}, we use the discrete equivalent of entropic safety, i.e., Shannon entropy, for our experiments because of the discrete representation of the mean-field distribution.~\cite[Proposition 8]{polyanskiy2016wasserstein} shows that Shannon entropy is Lipschitz continuous with respect to the scaled Wasserstein 1-distance, i.e., with respect to $\frac{1}{n}W_1(\cdot)$ known as Ornstein's distance. 
\end{remark}

\subsection{Distribution Similarity} \label{apx:distribution_similarity}
We can define safety by preventing the mean-field distribution $\mu$ from diverging from a prior distribution $\nu_0$ by quantifying the allowed dissimilarity between the two distributions
\begin{align*}
    h_{C}(\mu; \nu_0) &{}:={}C - D(\mu, \nu_0),
\end{align*}
where $C \geq 0$ and $D:\pp(\s)\times\pp(\s)\rightarrow\R_{\geq 0}$ is the distance function between probability measures. Commonly used distances are Wasserstein p-distance for $p{}\geq{} 1$ and $f$-divergences such as KL-divergence, Hellinger distance, and total variation distance.

A concrete example of a distance function $D(\cdot)$ that satisfies \Cref{asm:safety_lipschitz} is Wasserstein 1-distance
\begin{align*}
    W_1^C(\mu, \nu_0) &{}:={}C - W_1(\mu, \nu_0) \\
    h_C(\mu) &{}:={}W_1^C(\mu, \nu_0).
\end{align*}

\cite{clement2008elementary} shows the triangle inequality of Wasserstein p-distance for probability measures on separable metric spaces.~\cite{dudley2018real} shows that Wasserstein 1-distance induces a metric space $(\pp(\s), W_1)$ over probability measures. The result now trivially follows from the reverse triangle inequality
\begin{align*} \label{eqn:similarity_lipschitz}
    |h_C(\mu) - h_C(\mu')| 
    &{}={}|W_1^C(\mu, \nu_0) - W_1^C(\mu', \nu_0)| \\
    &{}={}|W_1(\mu, \nu_0) - W_1(\mu', \nu_0)| \\
    &{}\leq{} W_1(\mu, \mu').
\end{align*}

\begin{remark}
    Avoiding risky distributions can be modeled by setting $W_1^{C}(\mu, \nu_0) {}={}W_1(\mu, \nu_0) - C$ with the proof of \Cref{asm:safety_lipschitz} being equivalent to the above.
\end{remark}

\textbf{Weighted safety constraints.}\label{rmk:weighted_safety}
In applications that require emphasis on certain regions of the state space, we can generalize the above safety constraints by introducing the weight function $w:\s\rightarrow \R_{\geq 0}$. We extend \Cref{apx:entropic_safety} to weighted-differential entropy by defining
$$H^{w, \epsilon}(\mu) {}:={}-\int_{\s} w(s)\log(\mu(s) {}+{} \epsilon)\mu(ds)$$
and \Cref{apx:distribution_similarity} by considering weighted Wasserstein 1-distance
$$W_1^{w}(\mu, \nu_0) {}:={}\sup_{f:\text{Lip}(f){}\leq{} 1} \int_\s w(s)f(s)(\mu-\nu_0)(ds).$$
Here, we use Kantorovic-Rubinstein dual definition of Wasserstein 1-distance
$$W_1(\mu, \nu_0) {}:={}\sup_{f:\text{Lip}(f){}\leq{} 1} \int_\s f(s)(\mu-\nu_0)(ds),$$
where $f:\s\rightarrow\R$ is a continuous function and $\text{Lip}(f)$ denotes the minimal Lipschitz constant for $f$.

\section{Implementation Details} \label{apx:implementation_details}
In this section, we provide additional details on the practical implementation of \algnme. In particular, we discuss \textit{Probabilistic Neural Network Ensemble} model~\cite{chua2018deep,lakshminarayanan2017simple} to implement the statistical model from \Cref{sec:statistical_model} in \Cref{apx:ensemble}, the hallucinated control reparametrization from \Cref{sec:proposed_method} in more detail in \Cref{apx:hallucinated_control}, and optimization methods to solve \Cref{eqn:calibrated_safe} in \Cref{apx:mf-algorithms}

\subsection{Probabilistic Neural Network Ensemble Model of Transitions}\label{apx:ensemble}
As discussed in \Cref{sec:statistical_model}, we take a model-based approach to handling unknown transitions $f$.
The representative agent learns the \textit{Statistical Model} (see \Cref{sec:statistical_model}) of the transitions from the observed transitions $\cup_{i=1}^{n-1}\D_i$ at the beginning of each episode $n$, where $\D_i=\{(z_{i,t},s_{i,t{}+{}1})\}_{t=0}^{T-1}$ with $z_{i,t}=(s_{i,t},\mu_{i,t},a_{i,t})$. We use \textit{Probabilistic Neural Network Ensemble}~\cite{chua2018deep,lakshminarayanan2017simple} that consists of $K$ neural networks parametrized by $\theta_k$ for $k \in \{1, \ldots, K\}$ (the episode index should be clear from the context so we omit it for the notation simplicity). Each neural network $f_{\theta_k}$ returns a mean vector, $\bM_{\theta_k}(z) \in \s \subseteq \R^p$, and a covariance function $\bSigma_{\theta_k}(z) \in \R^{p \times p}$, that represents the aleatoric uncertainty. We further assume diagonal covariance functions. These outputs then form Gaussian distributions from which new states are sampled, i.e., $s_{t{}+{}1} \sim \mathcal{N}(\bM_{\theta_k}(z_t), \bSigma_{\theta_k}(z_t))$ for $t {}={}0,\ldots,T-1$. The models are trained with the negative log-likelihood loss function (NLL), $L(\theta) {}={}- \sum_{\D_{1:n-1}} \log\p(s_{t{}+{}1} | z_t)$ as described in~\cite{lakshminarayanan2017simple}.
The ensemble means, and the aleatoric and epistemic uncertainties are then estimated as follows
\begin{align*}
    \bM_{n-1}(\cdot) &{}={}\frac{1}{K} \sum_{k=1}^K \bM_{\theta_k}(\cdot) \\
    \bSigma_{n-1}^e(\cdot) &{}={}\frac{1}{K-1} \sum_{k=1}^K (\bM_{\theta_k}(\cdot) - \bM_{n-1}(\cdot))(\bM_{\theta_k}(\cdot) - \bM_{n-1}(\cdot))^\top \\
    \bSigma_{n-1}^a(\cdot) &{}={}\frac{1}{K} \sum_{k=1}^K \bSigma_{\theta_k}(\cdot),
\end{align*}
where $\bM_{n-1}(\cdot)$ is the ensemble prediction for the mean, while $\bSigma_{n-1}^a(\cdot)$ and $\bSigma_{n-1}^e(\cdot)$ denote the aleatoric and epistemic uncertainty estimates, respectively. Note that the epistemic uncertainty estimate $\bSigma_{n-1}^e(\cdot)$ is used to construct the calibrated model (see \Cref{sec:statistical_model} and \Cref{asm:calibrated_model}).

Even though Gaussian Processes (GPs) are proven to be calibrated under certain regularity assumptions~\cite{Abbasi-Yadkori2011RegretSystems, srinivas2009gaussian}, we chose probabilistic neural network ensemble model due to their better practical performance, i.e., scalability to higher dimensions and larger datasets. The disadvantage of the probabilistic neural network ensemble is that, unlike GPs, it does not guarantee the calibrated model (see \Cref{asm:calibrated_model}). Nevertheless, it can be recalibrated using one-step ahead predictions and temperature scaling as shown in~\cite{malik2019calibrated}. We note that in our experiment (see \Cref{sec:experiments}), such recalibration was not needed, and the above-defined mean and epistemic uncertainty outputs were sufficiently accurate for training \algnmespace (see \Cref{fig_apx:swarm_transitions}).

\subsection{Hallucinated Control Implementation Trick}\label{apx:hallucinated_control}
In \Cref{sec:proposed_method}, we introduced \algnme, a model that optimizes over the confidence set of transitions $\F_{n-1}$ and admissible policy profiles $\Pi$ (see \Cref{eqn:calibrated_short}). Unfortunately, optimizing directly over the function space is usually intractable since $\F_{n-1}$ is not convex, in general,~\cite{dani2008stochastic}. Thus, to make the optimization tractable, we describe a \textit{hallucinated control} trick, which leads to a practical reformulation (see \Cref{eqn:calibrated_safe}). The structure in $\F_{n-1}$ allows us to parametrize the problem and use gradient-based optimization to find a policy profile $\bpi_n^*$ at every episode $n$. Namely, we use the mean-field variant of an established approach known as \textit{Hallucinated Upper Confidence Reinforcement Learning} (H-UCRL)~\cite{curi2020efficient, moldovan2015optimism, pasztor2021efficient}. We introduce an auxiliary function $\eta:\zz \rightarrow [-1, 1]^p$, where $p$ is the dimensionality of the state space $\s$, to define hallucinated transitions \looseness=-1
\begin{equation} \label{eqn_apx:hallucinated_transitions}
    \tf_{n-1}(z) {}={}\bM_{n-1}(z) {}+{} \beta_{n-1}\bSigma_{n-1}(z)\eta(z), 
\end{equation}
where $\bM_{n-1}(\cdot)$ and $\bSigma_{n-1}(\cdot)$ are estimated from the past observations collected until the end of the previous episode $n-1$. Notice that $\tf_{n-1}$ is calibrated for any $\eta(\cdot)$ under \Cref{asm:calibrated_model}, i.e., $\tf_{n-1} \in \F_{n-1}$. \Cref{asm:calibrated_model} further guarantees that every function $\tf_{n-1}$ can be expressed in the auxiliary form in \Cref{eqn_apx:hallucinated_transitions}
\begin{equation*}
\begin{split}
    &\forall \tf_{n-1} \in \F_{n-1} \; \exists \eta:\zz \rightarrow [-1, 1]^p \; \text{such that} \\
    {}&\tf_{n-1}(z) {}={}\bM_{n-1}(z) {}+{} \beta_{n-1}\bSigma_{n-1}(z)\eta(z), \; \forall z \in \zz.
\end{split}
\end{equation*}
Furthermore, note that, for a fixed individual policy $\pi$, the auxiliary function $\eta(z) {}={}\eta(s, \mu, \pi(s, \mu)) {}={}\eta(s, \mu)$ has the same functional form as the policy $\pi$. This turns $\eta(\cdot)$ into a policy that exerts \textit{hallucinated control} over the epistemic uncertainty of the confidence set of transitions $\F_{n-1}$~\cite{curi2020efficient}. The reformulation of the optimization problem in \Cref{eqn:calibrated_safe} allows us to optimize over parametrizable functions (e.g., \textit{neural networks}) $\bpi$ and $\eta(\cdot)$ using gradient-based methods on the functions' parameters. Notice that the shared functional form of $\bpi$ and $\eta(\cdot)$ allows us to conveniently represent them by a single neural network. Further, note that the parametrization of $\eta(\cdot)$ must be sufficiently flexible not to restrict the space of $\tf_{n-1}$.
In \Cref{apx:mf-algorithms}, we provide several algorithms that can be used to solve this optimization problem.

\subsection{Optimization Methods -- MF-BPTT and MF-DDPG} \label{apx:mf-algorithms}

In this section, we describe two algorithms to solve the optimization problem in \Cref{eqn:calibrated_safe}. Namely, in \Cref{apx:mf-bptt}, we outline the key steps to apply the mean-field variant of the \textit{Back-Propagation-Through-Time} (BPTT) when a differentiable simulator is available and, in \Cref{apx:ddpg}, we describe the mean-field variant of the \textit{Deep Deterministic Policy Gradient} (DDPG)~\cite{lillicrap2015continuous} algorithm appropriate for non-differentiable simulators. \looseness=-1

\subsubsection{Mean-Field Back-Propagation-Through-Time (MF-BPTT)} \label{apx:mf-bptt}
\hfill\newline
Mean-Field Back-Propagation-Through-Time (MF-BPTT) assumes access to a differentiable simulator that returns a policy rollout given transition and policy functions. In each episode $n$, the representative agent initializes the policy profile $\bpi_n$, the auxiliary function $\eta(\cdot)$, and the estimated transitions $\tf_{n-1}$ using the mean $\bM_{n-1}(\cdot)$ and covariance $\bSigma_{n-1}(\cdot)$ functions according to \Cref{eqn:hallucinated_transitions}. Then, the representative agent repeatedly calls the simulator with inputs $\bpi_n$ and $\tf_{n-1}$ that returns the episode reward defined in \Cref{eqn:log-barrier-objective} as a differentiable object. After each policy rollout, a gradient ascent step is carried out on the parameters of $\bpi_n$ and $\eta(\cdot)$ before calling the simulator again. To simplify the notation, we overload the notation $\bpi_n^\psi$ with the combination of the two policy functions $\bpi_n$ and $\eta(\cdot)$, i.e., $\bpi_n^\psi {}={}(\bpi_n^\psi, \eta^\psi)$ where the superscript $\psi$ represent the parameters of both functions. We outline the described steps in \Cref{alg:mf_bptt}. During the optimization phase (Line 3 in \Cref{alg:learning_protocol}), we optimize both functions, $\bpi_n$ and $\eta(\cdot)$, jointly. However, during the execution (Line 4 in \Cref{alg:learning_protocol}), we only use the outputs corresponding to the policy profile $\bpi_n$.
The main distinction in \Cref{alg:mf_bptt} compared to traditional BPTT lies inside the simulator that has to simulate the mean-field distribution $\tmu_{n,t}$ for $t=1,\dots,T$ for each parameter update and calculate gradients with respect to this time dependency as well. More details on the implementation used for our experiments reported in \Cref{sec:experiments} are provided in \Cref{apx:vehicle_learning_protocol} and \Cref{apx:swarm_learning_protocol}. \looseness=-1

\begin{algorithm}[!t]
    \caption{Mean-Field Back-Propagation-Through-Time}
    \begin{algorithmic}[1]
        \Require Safety constraint $h_C(\cdot)$, calibrated transitions $\tf_{n-1}$ represented by $\bM_{n-1}(\cdot)$ and $\bSigma_{n-1}(\cdot)$, initial mean-field distribution $\mu_0$, known reward $r(\cdot)$, constants $C_{n,t}$, safety Lipschitz constant $L_h$; hyperparameter $\lambda$, number of epochs $K$, number of rollout steps $T$
    \State Initialize $\bpi^{\psi} {}={}(\pi^\psi_0,\ldots,\pi^\psi_{T-1})$
    \For{$k {}={}1,\ldots,K$}
        \State Initialize $\tmu_{0} \gets \mu_0$, $\ts_{0} \sim \mu_0$
        \State Initialize $r \gets 0$
        \For{$t {}={}0,\ldots,T-1$}
            \State $\ta_{t} \gets \pi^{\psi}_{t} (\ts_{t}, \tmu_{t})$
            \State $\ts_{t{}+{}1} \gets \tf_{n-1}(\ts_{t}, \tmu_{t}, \ta_{t}) {}+{} \epsilon_{t}$
            \State $\Tilde{\mu}_{t{}+{}1} \gets U(\tilde{\mu}_{t}, \pi^{\psi}_{t}, \tf_{n-1})$
        \EndFor
        \State Update $\psi$ with gradient ascent 
        \State $\quad\;\,\nabla_{\psi} \sum_{t=0}^{T-1} r(\ts_{t}, \Tilde{\mu}_{t}, \ta_{t}) {}+{} \lambda \log(h_C(\tmu_{t{}+{}1}) - L_{h}C_{n,t{}+{}1})$ 
    \EndFor
    \Ensure $\bpi_n^{\psi} \gets \bpi^{\psi}$
    \end{algorithmic}
    \label{alg:mf_bptt}
\end{algorithm}

\subsubsection{Mean-Field Deep Deterministic Policy Gradient (MF-DDPG)} \label{apx:ddpg}

In this section, we adopt the \textit{Deep Deterministic Policy Gradient} (DDPG) algorithm~\cite{lillicrap2015continuous} to the MFC. DDPG is a model-free actor-critic algorithm, hence, it can optimize \Cref{eqn:calibrated_safe} without the assumption of a differentiable simulator. However, it can not be applied directly to the problem because the Q-value for $\tilde{z}_{n,t} =(\ts_{n,t}, \tmu_{n,t}, \ta_{n,t})$ is ambiguous. The important insight here is that the value of a certain state $\ts_{n,t}$ of the environment reflects the whole population, i.e., the expected reward over the remainder of an episode for a given $\tmu_{n,t}$ if every agent in every state $\ts_{n,t}$ chooses actions following $\pi_{n,t}$. In essence, the Q-value is a function of $\tmu_{n,t}$ and $\pi_{n,t}$ and not $\tilde{z}_{n,t}$.

To overcome this issue, we introduce the \textit{lifted mean-field Markov decision process} (MF-MDP) similarly to~\cite{carmona2019model, Gu2019DynamicMFCs, Gu2021Mean-FieldAnalysis, Motte2019Mean-fieldControls, pasztor2021efficient}. First, we rewrite the reward function as a function of the mean-field distribution and the policy, i.e.,
\begin{equation*}
    \tilde{r}(\tmu_{n,t}, \pi_{n,t}) {}={}\int_\s r(s, \tmu_{n,t}, \pi_{n,t}(s, \tmu_{n,t}))\tmu_{n,t}(ds).
\end{equation*}
To simplify the notation, we overload the notation $\bpi_n^\psi {}={}(\bpi_n^\psi, \eta^\psi)$ as described in \Cref{apx:mf-bptt}. Then, we restate \Cref{eqn:calibrated_safe} as
\begin{subequations} \label{eqn:calibrated_safe_mfmdp}
    \begin{align}
    \psi^* {}={}\argmax_{\psi} \sum_{t=0}^{T-1} &\tilde{r}(\tmu_{n,t}, \pi_{n,t}^\psi) {}+{} \lambda\log(h_C(\tmu_{n,t{}+{}1}) - L_h C_{n,t{}+{}1}) \label{eqn:log-barrier-objective_mfmdp} \\
    \text{subject to}\quad \tf_{n-1}(\tz) &{}={}\bM_{n-1}(\tz) {}+{} \beta_{n-1}\bSigma_{n-1}(\tz)\eta^\psi(\tz) \\
    \tmu_{n,t{}+{}1} &{}={}U(\tmu_{n,t}, \pi_{n,t}^\psi, \tf_{n-1}),
\end{align}
\end{subequations}
where $\tmu_{n,0} {}={}\mu_0$ for every $n$. The MF-MDP formulation in \Cref{eqn:calibrated_safe_mfmdp} turns the MFC in \Cref{eqn:calibrated_safe} into a Markov Decision Process on the state space of $\mathcal{P}(\s)$ and action space $\{\pi: \s \times \mathcal{P}(\s) \to \mathcal{A}\}$ with deterministic transition function $U(\cdot)$. We define the Q-value as follows
\begin{equation*}
    Q_{n,t}(\tmu_{n,t}, \bpi_n^\psi) {}={}\sum_{j=t}^{T-1}\tilde{r}(\tmu_{n,t}, \pi_{n,t}^\psi).
\end{equation*}
The Q-function above is parameterized by $\theta$ and denoted as $Q^\theta_{n,t}$ for episode $n$. The DDPG algorithm can now be stated for the lifted MF-MDP problem as outlined in \Cref{alg:ddpg}. The main learning loop consists of alternating updates to the policy $\bpi^\psi_n$ and the critic $Q^\theta_n$. The most recent version of the policy is executed in the environment to collect more transitions into the replay buffer.

Notice that \Cref{alg:ddpg} uses a model-free approach to optimize the objective, which makes the exploration of particular importance. In comparison to traditional MDPs, MF-MDPs usually have a very constrained set of highly rewarding distributions, and most distributions offer poor rewards, which makes the exploration even more important.
This is further complicated by what randomized actions imply in this scenario. In a traditional MDP, we can often assume, for instance, that executing random actions for a fixed number of initial steps would help in finding diversity in the reward space.
This is not the case in MF-MDPs -- depending on the granularity of the discretization that we use to represent probability distributions, we can expect the mean-field distribution to stay fairly stable.
Similarly to~\cite{carmona2019model}, we might alleviate this issue by Gaussian mean-field initialization in each episode. Note, however, that this may not necessarily be appropriate in safety-constrained settings, as the initial mean-field distribution is expected to be safe. On top of that, we can add exploration noise to the actions obtained via the policy. Alternatively, we could add noise to the parameters of the policy network for a more consistent approach as in~\cite{plappert2017parameter}.

\begin{algorithm}[!t]
    \caption{Mean-Field Deep Deterministic Policy Gradient}
    \label{alg:ddpg}
    \begin{algorithmic}[1]
    \Require Safety constraint $h_C(\cdot)$, calibrated transitions $\tf_{n-1}$ represented by $\bM_{n-1}(\cdot)$ and $\bSigma_{n-1}(\cdot)$, initial mean-field distribution $\mu_0$, known expected reward $\hat{r}(\cdot)$, constants $C_{n,t}$; safety Lipschitz constant $L_h$, hyperparameter $\lambda$, number of epochs $K$, number of rollout steps $T$, mini-batch size $B$, learning rate $\alpha$
    \State Initialize $\bpi^{\psi} {}={}(\pi^\psi_0,\ldots,\pi^\psi_{T-1})$ 
    \State Initialize $Q^{\theta}{}={}(Q^\theta_0,\ldots,Q^\theta_{T-1})$
    \State Initialize $\theta' \gets \theta$, $\psi' \gets \psi$
    \State Initialize replay buffer $R \gets \emptyset$
    \For{$k {}={}1,\ldots,K$}
        \State Initialize $\tmu_{0} \gets \mu_0$
        \For{$t {}={}0,\ldots,T-1$}
            \State $\tmu_{t{}+{}1} \gets U(\tmu_{t}, \pi^{\psi}_{t}, \tf_{n-1})$
            \State $c_{t} \gets \hat{r}(\tmu_{t}, \pi^{\psi}_t) {}+{} \lambda\log(h_C(\tmu_{t+1}) - L_{h}C_{n,t+1})$
            \State $R \gets R \cup \{(\tmu_{t}, c_{t}, \tmu_{t{}+{}1})\}$
            \State Sample a mini-batch of $B$ random transitions $\{(\tmu_{i}, c_{i}, \tmu_{i{}+{}1})\}_{i=1}^B$ from $R$
            \For{$i=1,\ldots,B$}
                \State $q_{i} \gets c_{i} {}+{} \gamma Q^{\theta'}_t(\tmu_{i{}+{}1}, \bpi^{\psi'}(\cdot,\tmu_{i{}+{}1}))$
            \EndFor
            \State Update $\theta$ with gradient descent $\nabla_{\theta} \frac{1}{B} \sum_i (q_i - Q^{\theta}_t(\tmu_{i}, \bpi^{\psi}(\cdot,\tmu_{i}))^2$
            \State Update $\psi$ with gradient ascent $\nabla_{\psi} \frac{1}{B} \sum_i Q^{\theta}_t(\tmu_i,\bpi^{\psi}(\cdot,\tmu_{i}))$
            \State $\theta' \gets \alpha \theta {}+{} (1-\alpha) \theta'$
            \State $\psi' \gets \alpha \psi {}+{} (1-\alpha) \psi'$
        \EndFor
    \EndFor
    \Ensure $\bpi_{n}^\psi \gets \bpi^{\psi'}$
    \end{algorithmic}
\end{algorithm}

\section{Experiments -- Vehicle Repositioning}\label{apx:vehicle_repositioning}
In this section, we provide further analysis, the motivation behind our modeling decisions, and details for making our experiments easily replicable. We use a private cluster with GPUs to run our experiments. \algnmespace and \ucrlspace under known transitions each used 15 minutes of one Intel Xeon Gold 511 CPU core, 32 GB of RAM and one Nvidia GeForce RTX 3090 GPU. Training \algnmespace and \ucrlspace under unknown transitions to produce results in \Cref{fig:repositioning_learning_curve}, \Cref{fig:repositioning_constraint_violations} and \Cref{fig_apx:vehicle_data_efficiency} had 24 hours of access to fifty Intel Xeon Gold 511 CPU cores, 64 GB of RAM and fifty Nvidia GeForce RTX 3090 GPUs during the batch execution necessary for training. The evaluation, i.e., generating the results for, e.g., \Cref{fig_apx:vehicle_finite_regime_scatter} and \Cref{fig_apx:vehicle_finite_regime_mu} took around 1 hour of one Intel Xeon Gold 511 CPU core, 64 GB of RAM and one Nvidia GeForce RTX 3090 GPU. The only computationally intensive evaluation task was for \Cref{fig_apx:vehicle_finite_regime_rewards} for more than 1 million agents. We had access to Xeon Gold 511 CPU core, 128-256 GB of RAM, for around 8 hours. The implementation was predominantly done in Python packages PyTorch~\cite{paszke2017automatic} and NumPy~\cite{harris2020array}.

Our experimental workflow has the following structure:
\begin{enumerate} \label{item:repositioning_experimental_procedure}
    \item Input data preprocessing
    \begin{itemize}
        \item Estimating the demand distribution for the service $\rho_0$
        \item Estimating passenger's trip destinations' likelihood mapping $\Phi(\cdot)$
    \end{itemize}
    \item Modeling assumptions, model parameters, and distributions' representation
    \begin{itemize}
        \item State-space, action space, noise, reward, safety constraint 
        \item Mean-field distribution representation
        \item Mean-field transition $U(\cdot)$
    \end{itemize}
    \item Executing model-based learning protocol (\Cref{alg:learning_protocol})
    \begin{itemize}
        \item Optimizing \Cref{eqn:calibrated_safe}
        \item Learning unknown transitions using probabilistic neural network ensemble, i.e., statistical estimators $\bM_{n-1}(\cdot)$ and $\bSigma_{n-1}(\cdot)$
    \end{itemize}
    \item Performance evaluation in the infinite regime
    \item Performance evaluation in the finite regime
\end{enumerate}

\subsection{Input Data Preprocessing} \label{apx:vehicle_preprocessing}
We consider vehicle trajectories collected in Shenzhen's extended city center with the geographical area spanning from 114.015 to 114.14 degrees longitude and from 22.5 to 22.625 degrees latitude. We have access to the trajectories of five full working weeks (Monday to Sunday) collected between 18\textsuperscript{th} January 2016 and 25\textsuperscript{th} September 2016. We restrict ourselves to evening peak hours between 19:00 and 22:00. We represent probability distributions by discretizing the space into $k\times k$ unit grid with $k=25$ where each cell, $C_{ij}=[\frac{i}{k},\frac{i{}+{}1}{k}\rangle \times [\frac{j}{k},\frac{j{}+{}1}{k}\rangle$ for $i,j \in \{0,\ldots,k-1\}$, represents a square neighborhood of around $550\times 550$ meters on the city map. 
We represent the service demand distribution $\rho_0 \in \pp(\s)$ as a $k\times k$ matrix where entries $[\rho_0]_{ij}, i,j \in \{0,\ldots,k-1\}$ represent a probability of a trip originating in the neighborhood $C_{ij}$. The probabilities are estimated as an average over the considered period and kept constant during each step $t$ of the learning protocol (see \Cref{apx:vehicle_learning_protocol}). To smooth out the demand distribution and remove possible noise in the raw data, we apply 2-dimensional median smoothing with a window equal to 3 and show the output in \Cref{fig:target_mu}. If the passenger's trip originates at the state $s\in\s$ the likelihood of its destinations is defined by the mapping $\Phi:\s\rightarrow\pp(\s)$ which we use in \Cref{apx:vehicle_modeling} to define sequential transitions by first executing passenger trips followed by vehicle repositioning (see \Cref{fig:trajectory}). We flatten the $k\times k$ space grid into $k^2$-dimensional vector and represent $\Phi(\cdot)$ as a $k^2\times k^2$ probability matrix, i.e., rows summing to one represent outgoing mass for each cell. The entries $[\Phi]_{ij}$ with $i,j\in\{0,\ldots,k^2-1\}$ represent the likelihood of a passenger's trip that originated in the neighborhood $i$ ending in the neighborhood $j$. The likelihoods are estimated as an average over the considered period and kept constant. \Cref{fig_apx:vehicle_destination_likelihood} shows $\Phi$ column average, i.e., the trip destinations' likelihood given an arbitrary trip origin.

\begin{figure}[hbt!]
\centering
\begin{subfigure}[t]{0.42\textwidth}
    \centering
    \includegraphics[width=\textwidth]{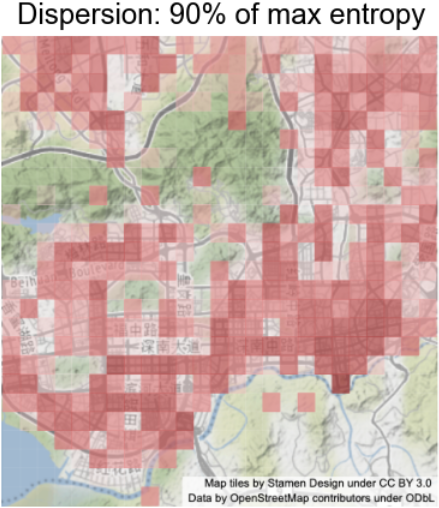}
\end{subfigure}
\begin{subfigure}[b]{0.05\textwidth}
    \centering
    \includegraphics[height=8.4cm]{figure2e.png}
\end{subfigure}
\caption{Passenger's trip destinations' likelihood during evening peak hours given an arbitrary trip origin (see \Cref{apx:vehicle_preprocessing}). We observe that the trips to most residential areas are almost equally likely and dispersed across the entire study region, which is consistent with the intuition that some residents commute back home (outside of the city center) after a work day, while others commute towards the city center for, e.g., leisure activities.}
\label{fig_apx:vehicle_destination_likelihood}
\Description{Passenger's trip destinations' likelihood during evening peak hours given an arbitrary trip origin. We observe that the trips to most residential areas are almost equally likely and dispersed across the entire study region, which is consistent with the intuition that some residents commute back home (outside of the city center) after a work day, while others commute towards the city center for, e.g., leisure activities.}
\end{figure}

\subsection{Modeling Assumptions and Model Parameters} \label{apx:vehicle_modeling}
We represent our area of interest as a two-dimensional unit square, i.e., the state-space $\s {}={}[0,1]^2$, and assume that vehicles can move freely using repositioning actions taken from $\A {}={}[-1,1]^2$. If the action takes a vehicle outside of the state-space borders, we project it back onto the border. Since the fleet is operating in a noisy traffic environment, repositioning usually cannot be executed perfectly. We model the noise $\epsilon_{n,t}\sim \text{TN}(0, \sigma^2 I_2)$ by a Gaussian with a known variance $\sigma^2$ truncated at the borders of $\s$ and $I_2$ is the $2\times 2$ unit matrix. We use a standard deviation $\sigma=0.0175$ to represent that the vehicle will be repositioned with a 68\% probability within a circle with a 240-meter radius of the desired location or with 95\% probability within a circle with a 480-meter radius. For simplicity, we assume that every passenger ride lasts fifteen minutes and that the repositioning is executed instantaneously. We assume that the representative agent (or multiple representative agents) reports the trajectories obtained during the interaction with the environment to the global controller at the end of the day. Therefore, we discretize our modeling horizon in fifteen-minute operational intervals, i.e., $T=12$, and each episode $n$ represents one day. The goal of the service provider is maximizing the profit, which correlates with the amount of satisfied demand. Therefore, we model the service provider goal of maximizing the coverage of the demand by the negative of the Kullback-Leibler divergence between the vehicles' distribution $\mu_{n,t}$ and demand for service $\rho_0$, i.e., $r(z_{n,t}) {}={}-D_{KL}(\rho_0 || \mu_{n,t})$. In other words, $r(\cdot)$ measures the similarity between vehicles' and demand distributions; the closer the distributions, the higher the profit. In practice, greedy profit maximization is often prevented by imposing service accessibility requirements by regulatory bodies. Due to the discrete representation of the mean-field distribution (discussed in the next paragraph), we use Shannon entropy to define the safety constraint
\begin{equation} \label{eq_apx:shannon_entropy}
    h_C(\mu_{n,t}) {}={}-\sum_{i,j} \log([\mu_{n,t}]_{ij})[\mu_{n,t}]_{ij} - C
\end{equation}
to enforce the service accessibility across all residential areas. Therefore, the optimization objective in \Cref{eqn:calibrated_safe} trades off between maximizing the profit $r(\cdot)$ and adhering to accessibility requirements imposed by $h_C(\cdot)$.
We use matrix $\Phi$ introduced \Cref{apx:vehicle_preprocessing} to model sequential transitions by first executing passenger trips followed by vehicle repositioning. Formally, the next state $s_{n,t{}+{}1} {}={}f(z_{n,t}) {}+{} \epsilon_{n,t}$ is induced by the unknown transitions $f(z_{n,t}) {}={}\text{clip}(s^\Phi_{n,t} {}+{} a_{n,t}, 0, 1)$, where $s^\Phi_{n,t}\sim\Phi(s_{n,t})$. Firstly, we find origin cell $i \in k^2$ such that $s_{n,t}$ resides in it and sample destination cell $j \in k^2$ given likelihoods defined in row $i$ of the probability matrix $\Phi$. Secondly, we determine the destination state $s^\Phi_{n,t}$ by uniform sampling from the destination cell $j$, which is a simplified model of the passenger preferences of the final destination. Notice that the controller determines repositioning actions given intermediate states $s^\Phi_{n,t}$ obtained after executing passenger trips.

\textbf{Mean-field transitions.} During the learning/training phase, we assume that the number of agents $m \rightarrow \infty$ and that these agents induce the mean-field distribution $\mu_{n,t}$ for episode $n$ and step $t$. But, one of the major practical challenges is implementing the mean-field transition function $U(\cdot)$ (see \Cref{eqn:mf_transition}). The main difficulties are representing the mean-field distribution $\mu_{n,t}$ and computing the integral in \Cref{eqn:mf_transition}. We use discretization to represent the mean-field distribution even though other representations, such as a mixture of Gaussians, are possible. Concretely, we represent the mean-field distribution as $\mu_{n,t} {}={}[\mu_{n,t}]_{ij}$ with $i,j \in \{0,\ldots,k-1\}$ with $k=25$ by associating the probability $[\mu_{n,t}]_{ij} {}={}\p \big\lbrack s_{n,t} \in C_{ij} \big\rbrack$ of the representative agent residing within each cell $C_{ij}$ during episode $n$ at step $t$. Note that the discrete representation of the mean-field distribution does not affect the state and action spaces, which remain continuous. The initial mean-field distributions, $\mu_{n,0}$ for every $n$, follow the uniform distribution which maximizes the Shannon entropy, ensuring the safety at the beginning of every episode $n$ at $t=0$, i.e., $h_C(\mu_{n,0}) {}\geq{} 0$. In vehicle repositioning, the mean-field transitions $U(\cdot)$ consist of two sequential steps induced by transitions $f$. Namely, first, the demand shifts the mean-field distribution, followed by the transition induced by the controller. Formally, the mean-field demand transition is computed as $\mu_{n,t}^\Phi=(\mu_{n,t}\cdot p)\times\Phi {}+{} \mu_{n,t}\cdot(1-p)$, where $\cdot$ denotes elementwise multiplication, $\times$ denotes matrix multiplication and $p=\min(1,\frac{\rho_0}{\mu_{n,t}})$ represents elementwise proportion of occupied vehicles. The mean-field controller transition requires computing the integral in \Cref{eqn:mf_transition} for which we use the discrete approximation given the points $c_{ij}$ uniformly chosen from cells $C_{ij}$ for $i,j\in\{0,\ldots,k-1\}$
 \begin{equation} \label{eq_apx:mf_discretized_transition}
     [\mu_{n,t{}+{}1}]_{ij} {}={}\sum_{k,l} \p[f(c_{kl}, \mu_{n,t}^\Phi,\pi_{n,t}(c_{kl}, \mu_{n,t}^\Phi)) {}+{} \epsilon_{n,t} \in C_{ij}][\mu_{n,t}^\Phi]_{kl},
 \end{equation}
for the episode $n$ and step $t$. We assume that the noise term $\epsilon_{n,t}$ is independent across episodes and steps as well as along the two dimensions while the truncation parameters are adjusted relative to the state space borders. Thus, we have the following
\begin{equation} \label{eq_apx:mf_discretized_probability}
\begin{split}
    &\p\left\lbrack f(c_{kl}, \mu^\Phi_{n,t},\pi_{n,t}(c_{kl}, \mu_{n,t}^\Phi)) {}+{} \epsilon_{n,t} \in C_{ij}\right\rbrack \\
    ={}& \p\left\lbrack f(c_{kl}, \mu^\Phi_{n,t},\pi_{n,t}(c_{kl}, \mu_{n,t}^\Phi))_x {}+{} \epsilon_{n,t,x} \in \left\lbrack\frac{i}{k},\frac{i{}+{}1}{k}\right\rangle\right\rbrack \\
    &\times \p\left\lbrack f(c_{kl}, \mu^\Phi_{n,t}, \pi_{n,t}(c_{kl}, \mu_{n,t}^\Phi))_y {}+{} \epsilon_{n,t,y} \in \left\lbrack\frac{j}{k},\frac{j{}+{}1}{k}\right\rangle\right\rbrack \\
    ={}& \left\lbrack\phi\left(\frac{i{}+{}1}{k}-f(c_{kl}, \mu^\Phi_{n,t},\pi_{n,t}(c_{kl}, \mu_{n,t}^\Phi))_x\right)\right. \\
    &\left. -\;\phi\left(\frac{i}{k}-f(c_{kl}, \mu^\Phi_{n,t},\pi_{n,t}(c_{kl}, \mu_{n,t}^\Phi))_x\right)\right\rbrack \\
    &\cdot \left\lbrack\phi\left(\frac{j{}+{}1}{k}-f(c_{kl}, \mu^\Phi_{n,t},\pi_{n,t}(c_{kl}, \mu_{n,t}^\Phi))_y\right)\right. \\
    &\left. -\;\phi\left(\frac{j}{k}-f(c_{kl}, \mu^\Phi_{n,t},\pi_{n,t}(c_{kl}, \mu_{n,t}^\Phi))_y\right)\right\rbrack, 
\end{split}
\end{equation}
where $\phi(\cdot)$ is the cumulative distribution function of truncated Gaussian $\text{TN}(0, \sigma^2)$.

\textbf{Mean-field transitions in the finite regime}\label{apx:vehicle_mf_finite_transitions}
In \Cref{apx:vehicle_evaluation_finite_regime}, we instantiate a finite number of vehicles $m < \infty$ to evaluate the policy performance in a realistic setting. We keep track of vehicles' states $s_t^{(l)}$ for every vehicle $l \in \{1,\ldots,m\}$ at steps $t {}={}0,\ldots,T$. In this setting, we approximate the mean-field transition $U(\cdot)$ with the normalized two-dimensional histogram $[\mu_{t{}+{}1}]_{ij}$ with bins defined by the cells $C_{ij}$ for  $i,j \in \{0,\ldots,k-1\}$ given vehicles' next states $s_{t{}+{}1}^{(l)} {}={}f(s_t^{(l)}, \mu_t^\Phi, \pi(s_t^{(l)},\mu_t^\Phi)) {}+{} \epsilon_t$.

\subsection{Model-Based Learning Protocol} \label{apx:vehicle_learning_protocol}
We use the learning protocol introduced in \Cref{alg:learning_protocol} to train \algnme. For hyperparameters of the model, see \Cref{tab_apx:vehicle_parameters_protocol}.

To optimize the objective in \Cref{eqn:calibrated_safe} in the subroutine in Line 3, we use MF-BPTT (see \Cref{apx:mf-bptt}). We parametrize the policy via a fully-connected neural network with two hidden layers of 256 neurons and Leaky-ReLU activations. The output layer returns the agents' actions using Tanh activation. We use Xavier uniform initialization~\cite{glorot2010understanding} to randomly initialize weights while we set bias terms to zero. We use 20,000 training epochs with the early stopping if the policy does not improve at least 0.5\% within 500 epochs. To prevent gradient explosion, we use L2-norm gradient clipping with max-norm set to 1. Note that in our experiments, a single neural network had enough predictive power to represent the whole policy profile $\bpi {}={}(\pi_0,\ldots,\pi_{T-1})$, but using $T$ networks, one for each individual policy $\pi_t$ is a natural extension. For further details about hyperparameters, see \Cref{tab_apx:vehicle_parameters_policy}. Additionally, note that we parametrize policy by a neural network with Lipschitz continuous activations (Leakly-ReLU and Tanh). In the case of the bounded neural network's weights, policy satisfies the Lipschitz continuity assumption (see \Cref{asm:policy_lipschitz}). In practice, we bound the weights via L2-regularization.

We estimate the confidence set of transitions $\F_{n-1}$ in the subroutine in Line 5 using a probabilistic neural network ensemble (see \Cref{apx:ensemble}). We use an ensemble of 10 fully-connected neural networks with two hidden layers of 16 neurons and Leaky-ReLU activations. The output layer returns the mean vector and variance vector (because of the covariance matrix diagonality assumption) of the confidence set; the mean uses linear activation, and the variance uses Softplus activation. We minimize the negative log-likelihood (NLL) for each neural network under the assumption of heteroscedastic Gaussian noise as described in~\cite{lakshminarayanan2017simple}. We randomly split the data from the replay buffer into a training set (90\%) and a validation set (10\%). We use 10,000 training epochs with the early stopping if the performance on the validation set does not improve for at least 0.5\% within 100 consecutive epochs. For further details about hyperparameters, see \Cref{tab_apx:vehicle_parameters_ensemble}.

\subsection{Performance Evaluation in the Infinite Regime} \label{apx:vehicle_evaluation_infinite_regime}
In this section, we extend the experiments presented in \Cref{sec:vehicle_repositioning}. We assume that the number of vehicles $m\rightarrow{}+{}\infty$ and show two important results. 
First, we show a conservative behavior of \algnmespace by training model with Shannon entropy safety constraint (see \Cref{rmk:shannon_entropy}) with the safety threshold $C$ set to $p=0.50$ of the maximum Shannon entropy, i.e., $C {}={}p \log(k^2)$. We then evaluate its performance against a much higher safety threshold, i.e., against the safety threshold induced by $p=0.85$. \Cref{fig_apx:vehicle_multi_constraint_violations} shows that the model never violates stricter safety constraint regardless of its training in a weaker setting. Second, we show the data efficiency of the learning protocol (\Cref{alg:learning_protocol}) by training models with access to one, five, and ten representative agents for data collection. \Cref{fig_apx:vehicle_data_efficiency} shows that the model under unknown transitions converges to the model trained under known transitions almost 6 times faster when using ten representative agents instead of one representative agent. It is a very useful result that can be utilized in many applications. For example, in most transportation applications, using tens or even hundreds of representative agents often does not cause cost-related issues. On the contrary, it might be more cost-effective to have ten representative agents for a month than one representative agent for a year.

\subsection{Performance Evaluation in the Finite Regime} \label{apx:vehicle_evaluation_finite_regime}
In this section, we assume a finite number of vehicles $m < {}+{}\infty$ and approximate the mean-field distribution $$\mu_t(s) {}={}\lim_{m \xrightarrow{} \infty} \frac{1}{m} \sum_{i=1}^m \I(s_t^{(i)} {}={}s)$$ with the empirical distribution as explained in \Cref{apx:vehicle_evaluation_finite_regime}. The policy profile $\bpi_n^*$ trained in the infinite regime is used to reposition each of $m$ individual vehicles in the fleet.
In \Cref{fig_apx:vehicle_finite_regime_rewards}, we show the relationship between model performance and the number of vehicles in the system. As expected, we observe that increasing the number of vehicles leads to better performance due to the increased precision of mean-field distribution $\mu_t$ approximation at step $t$. By performing 100 randomly initialized runs for each of various fleet sizes, we see that \algnmespace learned under unknown transitions and applied in the finite regime converges to the solution achieved in the infinite regime under known transitions. Furthermore, the model performs very well already for a fleet of 10,000 vehicles, which is in the order of magnitude of the fleet size that operates in Shenzhen. Concretely, in 2016, the fleet had around 17,000 vehicles, with an increasing trend. We also observe that in the finite regime, the difference in performance between \algnmespace trained under known and unknown transitions becomes insignificant. To showcase the practical usefulness of the algorithm, in \Cref{fig_apx:vehicle_finite_regime_scatter}, we instantiate 10,000 vehicles and display their positions after the final repositioning at step $T=12$. We observe that the majority of vehicles are repositioned to areas of high demand. At the same time, some of them are sent to residential zones in the northwest and northeast to enforce accessibility. It is important to note that some vehicles are repositioned to aquatic areas and areas without infrastructure due to two reasons. First, our model guarantees global safety without explicit guarantees for individual/local safety. Notice that undesirable areas might be avoided by safety constraint shaping, e.g., by setting the weight function for these areas to zero as discussed in \Cref{rmk:weighted_safety}. Second, the model loses some of its accuracy due to the finite regime approximation errors. In \Cref{fig_apx:vehicle_finite_regime_mu}, we observe only a slight decrease in dispersion when \algnmespace finite regime approximation is compared to the infinite regime performance. To conclude, we showcase the potential of \algnmespace for vehicle repositioning in the finite regime, which might be a positive signal for real-world practitioners.

\section{Experiments -- Swarm Motion}\label{apx:swarm_motion}
In this section, we extend the swarm motion experiments discussed in \Cref{sec:swarm_motion} and complement the vehicle repositioning experiments elaborated in \Cref{sec:vehicle_repositioning} and \Cref{apx:vehicle_repositioning}. For this experiment, we use the same private cluster and approximately the same amount of computational resources as reported in \Cref{apx:vehicle_repositioning}.

\subsection{Modeling Assumptions and Model Parameters} \label{apx:swarm_modeling}
We model the state space $\s$ as the unit torus on the interval $[0,1]$ and set the action space as the interval $\A {}={}[-7, 7]$ due to the knowledge of the range of actions from the continuous-time analytical solution~\cite{almulla2017two}. We approximate the continuous-time swarm motion by partitioning unit time into $T=100$ equal steps of length $\Delta t {}={}1{}/T$. The next state $s_{n,t{}+{}1} {}={}f(z_{n,t}) {}+{} \epsilon_{n,t}$ is induced by the unknown transitions $f(z_{n,t}) {}={}s_{n,t} {}+{} a_{n,t}\Delta t$ with $\epsilon_{n,t} \sim \text{N}(0, \Delta t)$ for all episodes $n$ and steps $t$. The reward function is defined by $r(z_{n,t}) {}={}\phi(s_{n,t}) - \frac{1}{2}a_{n,t}^2 - \log(\mu_{n,t})$, where the first term $\phi(s) {}={}2 \pi^2 (\sin(2\pi s) - \cos^2(2\pi s)) {}+{} 2\sin(2\pi s)$ determines the positional reward received at the state $s$ (see \Cref{fig_apx:positional_reward}), the second term defines the kinetic energy penalizing large actions, and the last term penalizes overcrowding. Note that the optimal solution for continuous time setting, $\Delta t \rightarrow 0$, can be obtained analytically. Namely, for the infinite time horizon, i.e., $T \rightarrow \infty$, we have
\begin{align} \label{apx:analytical_solution}
\begin{split}
    \pi^*(s,\mu) &{}={} 2\pi\cos(2\pi s) \\
    \mu^*(s) &{}={} \frac{e^{2\sin(2\pi s)}}{\int_\s e^{2\sin(2\pi s')}ds'},
\end{split}
\end{align}
where $\pi^*$ and $\mu^*$ form an ergodic solution satisfying $\mu^* {}={} U(\mu^*,\pi^*,f)$. We use $\mu^*$ as a benchmark but note that it might no longer be an optimal solution in the discrete-time setting. Therefore, discrete-time solutions obtained under known transitions serve as a good benchmark for \algnmespace performance under unknown transitions. To control overcrowding, we use the Shannon-entropic constraint introduced in \Cref{eq_apx:shannon_entropy} instead of having the overcrowding penalty term $\log(\mu_{n,t})$ in the reward. As discussed in \Cref{apx:vehicle_modeling}, Shannon entropy is used because of the discrete mean-field distribution representation. Since higher entropy translates into less overcrowding, we can upfront determine and upper-bound the acceptable level of overcrowding by setting a desirable threshold $C$. Similarly to the discussion in \Cref{apx:vehicle_modeling}, we represent the mean-field distribution by discretizing the state space into $k = 100$ uniform intervals and assigning the probability of the representative agent residing within each of them. To compute the mean-field transitions $U(\cdot)$, we use one-dimensional equivalent of \Cref{eq_apx:mf_discretized_transition} and \Cref{eq_apx:mf_discretized_probability}. We set the safety threshold $C$ as a proportion $p \in [0,1]$ of the maximum Shannon entropy \Cref{eq_apx:shannon_entropy}, i.e., $C{}={}p\log(k)$. We initialize safe mean-field distributions $\mu_{n,0}$ as uniform distributions since they maximize Shannon entropy, which makes them safe for every threshold $C$.

\subsection{Model-Based Learning Protocol} \label{apx:swarm_learning_protocol}
We follow the same procedure as described in \Cref{apx:vehicle_learning_protocol} with the hyperparameters from \Crefrange{tab_apx:swarm_parameters_protocol}{tab_apx:swarm_parameters_ensemble}. The only difference compared to \Cref{apx:vehicle_learning_protocol} is the increase in the complexity of the computational graph because of the high number of steps $T$. Therefore, we use batch normalization~\cite{ioffe2015batch} to prevent vanishing gradients.

\subsection{Performance Evaluation} \label{apx:swarm_evaluation}
In this section, we complement the results shown in \Cref{sec:swarm_motion}. We first visualize the positional reward $\phi(\cdot)$ in \Cref{fig_apx:positional_reward} for the ease of interpretation of the obtained results. We see that the reward has two local maxima, but due to a significant difference in their value, unconstrained benchmarks ignore the lower maxima. On the other hand, \algnmespace for $p=0.95$, and to a certain extent for $p=0.5$, take advantage of it by reducing the kinetic energy in the neighborhood of lower maxima as shown in \Cref{fig:swarm_policies}. 
\begin{figure}[hbt!]
\centering
\includegraphics[width=\columnwidth]{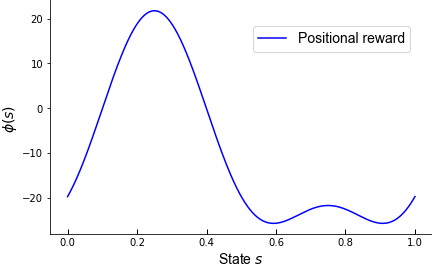}
\caption{Swarm motion positional reward $\phi(\cdot)$.}
\label{fig_apx:positional_reward}
\Description{Swarm motion positional reward.}
\end{figure}
In \Cref{fig_apx:swarm_progression}, we show the mean-field distributions progression over time guided by policies learned by \algnmespace for $p=0.5$ and $p=0.95$. We observe that for $p=0.5$, we reach near-stationary distribution after only 10 steps (see \Cref{fig_apx:swarm_progression_50pct}), i.e., the distribution remains the same until the algorithm terminates at $T=100$. For $p=0.95$, we reach stationarity even faster, as shown in \Cref{fig_apx:swarm_progression_95pct}. The learning process presented in \Cref{fig:swarm_learning_curve} is explained by the reduction of the epistemic uncertainty in the estimated transitions $\tf_{n-1}$. Before the first episode $n=1$, the statistical model is estimated only from trajectories collected by randomly initialized policy $\pi_0$. Due to the high epistemic uncertainty in regions that random policy did not explore, upper-confidence hallucinated transitions \Cref{eqn:hallucinated_transitions} do not approximate well true transitions (see \Cref{fig_apx:swarm_transitions_episode1}). By episode $n=5$, the model already has a good approximation of the transitions (see \Cref{fig_apx:swarm_transitions_episode5}), while at episode $n=50$, the transitions are known with near-certainty (see \Cref{fig_apx:swarm_transitions_episode50}). These results coincide with the observation in \Cref{fig:swarm_learning_curve} where around episode $n=50$, \algnmespace starts obtaining the results as if the transitions were known. Note that we implement the toroidal state space on $\s=[0,1]$ by assuming a sufficiently large extension, e.g., $[-1,2]$, of the interval over its borders such that a new state resulting from any possible action is captured with high probability. The new state is then mapped back to interval $[0,1]$ using the modulo operation.
For completeness of the analysis, in \Cref{fig_apx:swarm_learning_curve_50pct}, we show that \algnmespace for $p=0.5$ converges to the result obtained under known transitions. 
\begin{figure}[hbt!]
\centering
\includegraphics[width=\columnwidth]{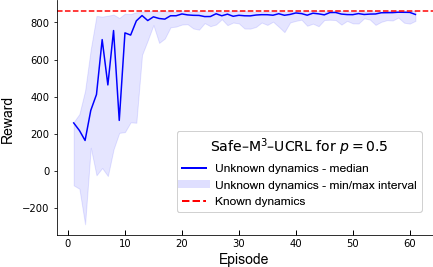}
\caption{\algnmespace learning curve for $p=0.5$ for swarm motion.}
\label{fig_apx:swarm_learning_curve_50pct}
\Description{Swarm motion learning curve for p equals 0.5.}
\end{figure}
We further validate the observation presented in \Cref{fig:swarm_distributions} that the constraint for $p=0.5$ results in similar overcrowding as the reward penalty term $-\log(\mu)$. Namely, in \Cref{fig_apx:swarm_constraint_violations_50pct}, we see that \algnmespace for $p=0.5$ and \ucrlspace with overcrowding penalty term satisfy the safety constraint $h_C(\mu) = 0.5\log(k)$ with a similar margin. 
\begin{figure}[hbt!]
\centering
\includegraphics[width=\columnwidth]{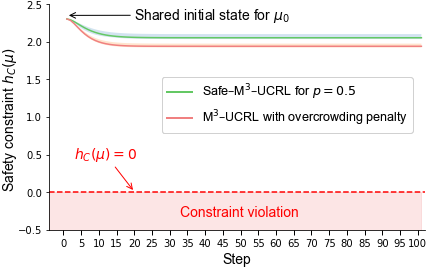}
\caption{Swarm motion safety for $p=0.5$.}
\label{fig_apx:swarm_constraint_violations_50pct}
\Description{Swarm motion constraint violation for p equals 0.5}
\end{figure}

\begin{figure*}[!t]
\centering
\begin{subfigure}[t]{0.49\textwidth}
    \centering
    \includegraphics[width=\textwidth]{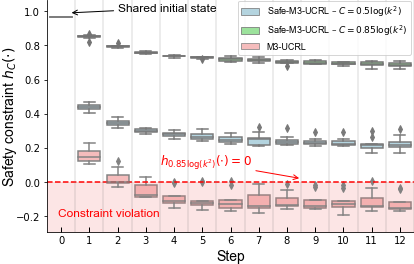}
    \captionsetup{justification=centering}
    \caption{\algnmespace conservative behavior}
    \label{fig_apx:vehicle_multi_constraint_violations}
\end{subfigure}
\hfill
\begin{subfigure}[t]{0.49\textwidth}
    \centering
    \includegraphics[width=\textwidth]{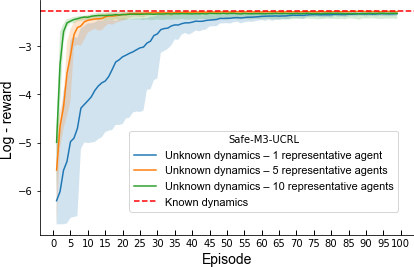}
    \captionsetup{justification=centering}
    \caption{Learning protocol data efficiency}
    \label{fig_apx:vehicle_data_efficiency}
\end{subfigure}
\caption{We showcase \algnmespace conservative behavior and data efficiency by training 10 randomly initialized policy profiles and statistical models where each setup uses the entropic safety constraint $h_C(\cdot){}\geq{} 0$. We set the safety threshold $C$ as a proportion $p$ of the maximum Shannon entropy, i.e., $C=p\log(k^2)$ with $k=25$. In (a), the policy profiles trained for satisfying $h_{0.5\log(k^2)}(\cdot){}\geq{} 0$ never violate $h_{0.85\log(k^2)}(\cdot){}\geq{} 0$, which shows the conservative behavior of our model. In (b), we show the data efficiency of the learning protocol (\Cref{alg:learning_protocol}) by comparing learning curves observed during training models that satisfy $h_{0.85\log(k^2)}(\cdot){}\geq{} 0$ when using one, five and ten representative agents (RA) for data collection. We see that the model trained with 1-RA converges to the performance of the model under known transitions in around 80 episodes, while it takes 25 and 15 episodes for 5-RA and 10-RA models, respectively. Note that we use log-reward to emphasize learning speeds on a comparable scale.}
\label{fig_apx:vehicle_multi_training}
\Description{Learning protocol conservative behavior and data efficiency.}
\end{figure*}

\begin{figure*}[hbt!]
\centering
\begin{subfigure}[t]{0.49\textwidth}
    \centering
    \includegraphics[width=\textwidth]{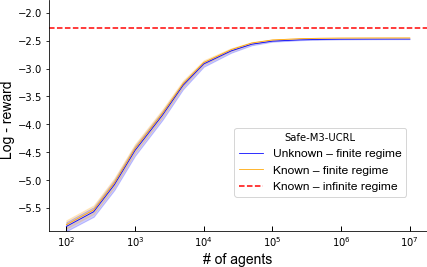}
    \captionsetup{justification=centering}
    \caption{\\ Objective from \Cref{eqn:calibrated_safe} achieved in the finite regime by approximating mean-field distribution with empirical distribution}
    \label{fig_apx:vehicle_finite_regime_rewards}
\end{subfigure}
\hfill
\begin{subfigure}[t]{0.49\textwidth}
    \centering
    \includegraphics[width=5cm]{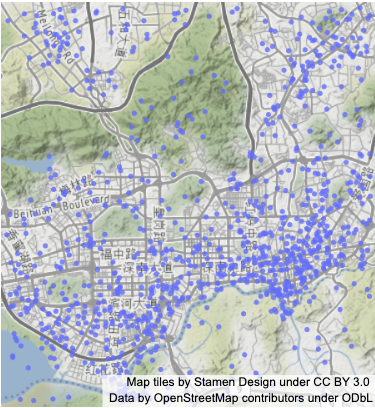}
    \captionsetup{justification=centering}
    \caption{\\ Vehicles' locations after repositioning action \\ in the finite regime}
    \label{fig_apx:vehicle_finite_regime_scatter}
\end{subfigure}
\caption{To showcase \algnmespace performance in the finite regime, we instantiate a finite number of vehicles, each following a policy profile $\bpi_n^*$ learned in the infinite regime and satisfying the entropic safety constraint $h_C(\cdot){}\geq{} 0$ for $C=0.85\log(k^2)$ with $k=25$. We perform 100 randomly initialized runs for each of the various fleet sizes. In (a), we see that increasing the number of vehicles leads to better performance due to the increased precision of mean-field distribution approximation. Further, we see that \algnmespace learned under unknown transitions and applied in the finite regime converges to the solution achieved in the infinite regime under known transitions. We also see that the curves for \algnmespace trained under unknown and known transitions almost overlap, i.e., the value of knowing transitions has little to no insignificance in the finite regime. In (b), we showcase the performance for the realistic number of vehicles operating in Shenzhen (10,000 to 20,000). We display the positions of a randomly chosen subset of 1,000 vehicles (out of 10,000) at step $T=12$ after the final repositioning. We observe that the majority of vehicles are repositioned to areas of high demand, while some of them are sent to residential zones in the northwest and northeast to enforce accessibility.}
\label{fig_apx:vehicle_finite_regime}
\Description{The proposed algorithm performance in finite regime.}
\end{figure*}

\begin{figure*}[hbt!]
\centering
\begin{subfigure}[t]{0.46\textwidth}
    \centering
    \includegraphics[width=\textwidth]{figure2d.png}
    \captionsetup{justification=centering}
    \caption{\\\algnmespace under unknown transitions in the infinite regime}
   \label{fig_apx:vehicle_safe_ucrl_mu_known}
\end{subfigure}
\hfill
\begin{subfigure}[t]{0.46\textwidth}
    \centering
    \includegraphics[width=\textwidth]{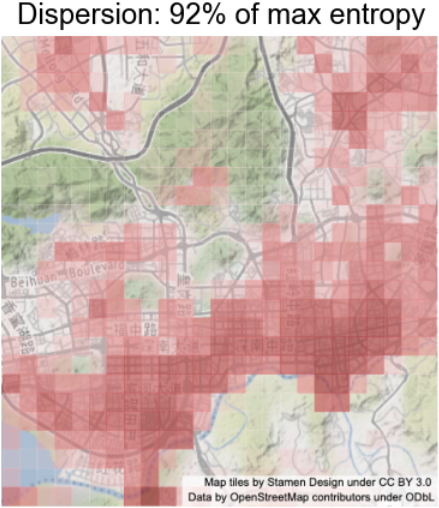}
    \captionsetup{justification=centering} 
    \caption{\\\algnmespace under unknown transitions in the finite regime}
    \label{fig_apx:vehicle_safe_ucrl_mu_unknown_finite}
\end{subfigure}
\hfill
\begin{subfigure}[b]{0.05\textwidth}
    \centering
    \includegraphics[height=9.2cm]{figure2e.png}
\end{subfigure}
\caption{We show the difference in the performance introduced by approximating mean-field distribution with the finite number of vehicles as described in \Cref{apx:vehicle_mf_finite_transitions}. In (a), we see that the dispersion of \algnmespace in the infinite regime is $p=0.96$ as elaborated in \Cref{fig:repositioning_mu_comparison}. In (b), the policy profile $\bpi_n^*$ trained in the infinite regime is used to control a fleet of 10,000 vehicles. We observe a slight decrease in the dispersion from $p=0.96$ to $p=0.92$ due to the mean-field distribution approximation errors.} \label{fig_apx:vehicle_finite_regime_mu}
\Description{Mean-field approximation error in the finite regime.}
\end{figure*}

\begin{figure*}[htb!]
\centering
\begin{subfigure}[t]{\textwidth}
    \centering
    \includegraphics[width=\textwidth]{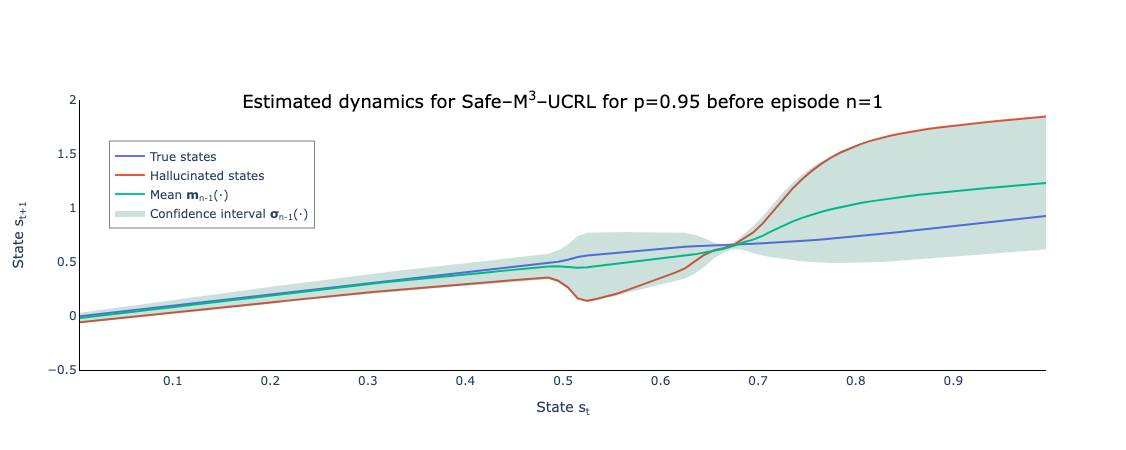}
    \captionsetup{justification=centering}
    \caption{}
    \label{fig_apx:swarm_transitions_episode1}
\end{subfigure}

\begin{subfigure}[t]{\textwidth}
    \centering
    \includegraphics[width=\textwidth]{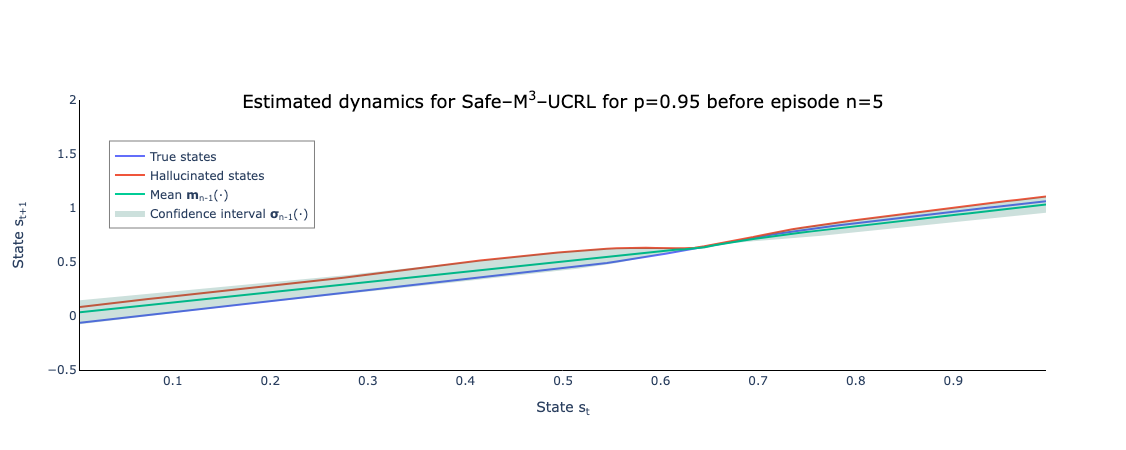}
    \captionsetup{justification=centering}
    \caption{}
    \label{fig_apx:swarm_transitions_episode5}
\end{subfigure}

\begin{subfigure}[t]{\textwidth}
    \centering
    \includegraphics[width=\textwidth]{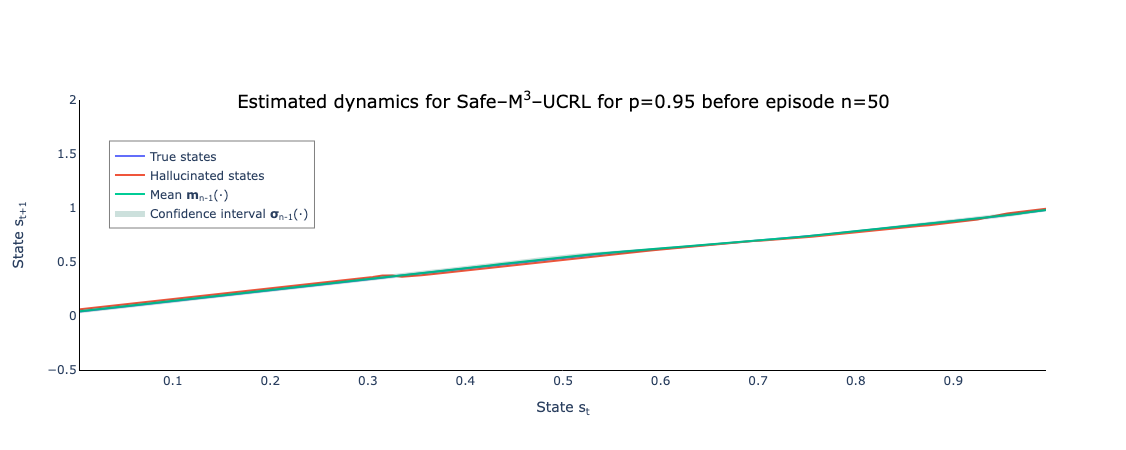}
    \captionsetup{justification=centering}
    \caption{}
    \label{fig_apx:swarm_transitions_episode50}
\end{subfigure}
\label{fig_apx:swarm_transitions}
\end{figure*}

\begin{figure*}[!t]
\centering
\begin{subfigure}[t]{0.49\textwidth}
    \centering
    \includegraphics[width=\textwidth]{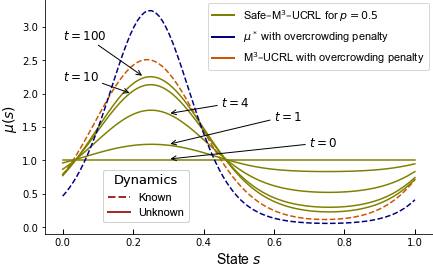}
    \captionsetup{justification=centering}
    \caption{}
    \label{fig_apx:swarm_progression_50pct}
\end{subfigure}
\hfill
\begin{subfigure}[t]{0.49\textwidth}
    \centering
    \includegraphics[width=\textwidth]{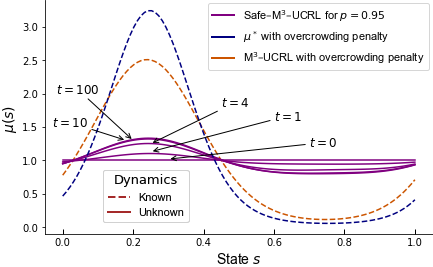}
    \captionsetup{justification=centering}
    \caption{}
    \label{fig_apx:swarm_progression_95pct}
\end{subfigure}
\caption{Mean-field distributions progression over time.}
\label{fig_apx:swarm_progression}
\Description{Mean-field distributions progression over time.}
\end{figure*}

\begin{table*}[hbt!]
\centering
\caption{Learning protocol hyperparameters for vehicle repositioning}
\label{tab_apx:vehicle_parameters_protocol}
\begin{tabular}{lrl}
\toprule
Hyperparameter  & Value & Description \\
\midrule
$N$ & 200 & Number of episodes \\
$T$ & 12 & Number of steps \\
$k$ & 25 & Number of discretization segments per axis  \\
$\sigma$ & 0.0175 & Standard deviation of the system noise  \\
$L_h$ & 0.1 & Lipschitz constant in \Cref{eqn:calibrated_safe}  \\
$\lambda$ & 1 & Log-barrier hyperparameter in \Cref{eqn:calibrated_safe} \\
\# of representative agents & 1 to 10 & By default 1, but in some experiments, we use more \\
\bottomrule
\end{tabular}
\end{table*}

\begin{table*}[hbt!]
\centering
\caption{Policy hyperparameters for vehicle repositioning}
\label{tab_apx:vehicle_parameters_policy}
\begin{tabular}{lrl}
\toprule
Hyperparameter & Value & Description \\
\midrule
\# of hidden layers & 2 &   \\
\# of neurons & 256 & Number of neurons per hidden layer \\
hidden activations & Leaky-ReLU & \\
output activation & Tanh & \\
$\alpha$ & $10^{-4}$ & Learning rate \\
$w$ & $5\cdot 10^{-4}$ & Weight decay \\
initialization & Xavier uniform & \\
bias initialization & 0 & \\
$n$ & 20,000 & Number of epochs \\
early stopping & 0.5\% & If not improved after 500 epochs \\
\bottomrule
\end{tabular}
\end{table*}

\begin{table*}[hbt!]
\centering
\caption{Probabilistic neural network ensemble hyperparameters for vehicle repositioning}
\label{tab_apx:vehicle_parameters_ensemble}
\begin{tabular}{lrl}
\toprule
Hyperparameter & Value & Description \\
\midrule
\# of ensemble members & 10 & \\
\# of hidden layers & 2 &   \\
\# of neurons & 16 & Number of neurons per hidden layer \\
hidden activations & Leaky-ReLU & \\
mean output activation & Linear & \\
variance output activation & Softplus & \\
$\alpha$ & $10^{-4}$ & Learning rate \\
$w$ & $5\cdot 10^{-4}$ & Weight decay \\
$\beta$ & 1 & \Cref{asm:calibrated_model} hyperparameter \\
initialization & Xavier uniform & \\
bias initialization & 0 & \\
$n$ & 10,000 & Number of epochs \\
early stopping & 0.5\% & If not improved for 100 consecutive epochs \\
train-validation split & 90\%-10\% & We use a validation set for early stopping \\
replay buffer size & 10-100 & By default 100 \\
batch size & 8-128 & Increasing with the replay buffer size \\
\bottomrule
\end{tabular}
\end{table*}

\begin{table*}[hbt!]
\centering
\caption{Learning protocol hyperparameters for swarm motion}
\label{tab_apx:swarm_parameters_protocol}
\begin{tabular}{lrl}
\toprule
Hyperparameter  & Value & Description \\
\midrule
$N$ & 200 & Number of episodes \\
$T$ & 100 & Number of steps \\
$k$ & 100 & Number of discretization segments per axis  \\
$\sigma$ & 1 & Standard deviation of the system noise  \\
$L_h$ & $10^{-4}$ & Lipschitz constant in \Cref{eqn:calibrated_safe}  \\
$\lambda$ & 15 & Log-barrier hyperparameter in \Cref{eqn:calibrated_safe} \\
\# of representative agents & 1 & \\
\bottomrule
\end{tabular}
\end{table*}

\begin{table*}[hbt!]
\centering
\caption{Policy hyperparameters for swarm motion}
\label{tab_apx:swarm_parameters_policy}
\begin{tabular}{lrl}
\toprule
Hyperparameter & Value & Description \\
\midrule
\# of hidden layers & 2 &   \\
\# of neurons & 16 & Number of neurons per hidden layer \\
hidden activations & Leaky-ReLU & \\
output activation & Tanh & \\
$\alpha$ & $5\cdot 10^{-3}$ & Learning rate \\
$w$ & $5\cdot 10^{-4}$ & Weight decay \\
initialization & Xavier uniform & \\
bias initialization & 0 & \\
$n$ & 50,000 & Number of epochs \\
early stopping & 0.5\% & If not improved after 100 epochs \\
\bottomrule
\end{tabular}
\end{table*}

\begin{table*}[hbt!]
\centering
\caption{Probabilistic neural network ensemble hyperparameters for swarm motion}
\label{tab_apx:swarm_parameters_ensemble}
\begin{tabular}{lrl}
\toprule
Hyperparameter & Value & Description \\
\midrule
\# of ensemble members & 10 & \\
\# of hidden layers & 2 &   \\
\# of neurons & 16 & Number of neurons per hidden layer \\
hidden activations & Leaky-ReLU & \\
mean output activation & Linear & \\
variance output activation & Softplus & \\
$\alpha$ & $5\cdot10^{-3}$ & Learning rate \\
$w$ & $5\cdot 10^{-4}$ & Weight decay \\
$\beta$ & 1 & \Cref{asm:calibrated_model} hyperparameter \\
initialization & Xavier uniform & \\
bias initialization & 0 & \\
$n$ & 10,000 & Number of epochs \\
early stopping & 0.5\% & If not improved for 30 consecutive epochs \\
train-validation split & 90\%-10\% & We use a validation set for early stopping \\
replay buffer size & 10,000 &  \\
batch size & 8-512 & Increasing with the replay buffer size \\
\bottomrule
\end{tabular}
\end{table*}

\end{document}